\documentclass[11pt]{article}

\usepackage{amsmath}
\usepackage{amssymb}
\usepackage{amsthm}
\usepackage{mathtools}
\usepackage{mathrsfs}

\usepackage{abstract}
\usepackage{appendix}
\usepackage{titlesec}

\usepackage{natbib}
\bibpunct{(}{)}{;}{a}{,}{,}

\usepackage{hyperref}
\hypersetup{
    colorlinks=true,
    linkcolor={blue},
    urlcolor={black},
    citecolor={blue},
    linktoc=all,
}

\usepackage[
    left=1in,
    right=1in,
    top=1in,
    bottom=1in,
    headheight=0pt,
    headsep=5pt
]{geometry}
\usepackage{caption}

\usepackage[shortlabels,inline]{enumitem}

\usepackage{array}
\usepackage{booktabs}

\usepackage{graphicx}

\usepackage{xparse}
\usepackage{xstring}

\makeatletter
\newtoks\version
\newtoks\institute
\renewcommand{\@maketitle}{%
    \newpage
    \null
    \vskip 2em%
    \begin{center}%
        {\LARGE\bfseries \@title \par}%
        \vskip 1.5em%
            {\large
                \lineskip .5em%
                \begin{tabular}[t]{c}%
                    \@author \\[1ex]
                \end{tabular}\par}%
        \the\institute%
        \vskip 0.5ex%
    \end{center}%
    \par
}
\makeatother

\titleformat{\section}[block]
{\normalfont\LARGE\bfseries}
{\thesection}{1em}{}
\titleformat{\subsection}[block]
{\normalfont\Large\bfseries}
{\thesubsection}{1em}{}
\titleformat{\subsubsection}[block]
{\normalfont\large\bfseries}
{\thesubsubsection}{1em}{}

\AtBeginDocument{
    \setlength{\jot}{8pt}
    \setlength{\abovedisplayskip}{8pt}
    \setlength{\belowdisplayskip}{8pt}
    \setlength{\abovedisplayshortskip}{6pt}
    \setlength{\belowdisplayshortskip}{6pt}

    \RequirePackage[flushmargin]{footmisc}
    \setlength{\footnotesep}{12pt}
}

\theoremstyle{plain}
\newtheorem{theorem}{Theorem}
\newtheorem{proposition}{Proposition}
\newtheorem{lemma}{Lemma}

\newtheorem{definition}{Definition}
\newtheorem{assumption}{Assumption}

\newtheorem*{remark}{\normalfont\textbf{Remark}}

\newcommand{\keywords}[1]{\vskip 2ex\par\noindent\normalfont{\bfseries Keywords:} #1}
\newcommand{\email}[1]{\href{mailto:#1}{\nolinkurl{#1}}}
\newcommand{\ROM}[1]{\uppercase\expandafter{\romannumeral #1}}

\newcommand{\secref}[1]{Section~\ref{#1}}

\newcommand{\thmref}[1]{Theorem~\ref{#1}}
\newcommand{\proref}[1]{Proposition~\ref{#1}}
\newcommand{\lemref}[1]{Lemma~\ref{#1}}

\newcommand{\defref}[1]{Definition~\ref{#1}}
\newcommand{\aspref}[1]{Assumption~\ref{#1}}

\newcommand{\neweqref}[2]{\textup{(\hyperref[#1]{\textbf{#2}})}}

\newenvironment{proofof}[1]{
    \begin{proof}[\normalfont\textbf{Proof of #1}]
        }{
    \end{proof}
}

\newcommand{\dd}{\mathrm{d}}
\newcommand{\ee}{\mathrm{e}}
\newcommand{\eqspace}{\hphantom{{}={}}}  

\newcommand{\bbR}{\mathbb{R}}

\newcommand{\calH}{\mathcal{H}}
\newcommand{\calL}{\mathcal{L}}

\newcommand{\calX}{\mathcal{X}}
\newcommand{\calY}{\mathcal{Y}}

\newcommand{\bsu}{\boldsymbol{u}}
\newcommand{\bsv}{\boldsymbol{v}}
\newcommand{\bsw}{\boldsymbol{w}}
\newcommand{\bsx}{\boldsymbol{x}}

\newcommand{\LE}{\left}

\newcommand{\RI}{\right}

\newcommand{\ind}[1]{\mathbf{1}_{ #1 }}

\makeatletter
\DeclarePairedDelimiter{\DVert}{\lVert}{\rVert}
\DeclarePairedDelimiter{\Dangle}{\langle}{\rangle}
\DeclarePairedDelimiter{\Dbrace}{\lbrace}{\rbrace}
\DeclarePairedDelimiter{\Dbrack}{\lbrack}{\rbrack}
\DeclarePairedDelimiter{\Dparen}{\lparen}{\rparen}
\newcommand{\norm}[1]{%
  \if@display
    \DVert*{#1}%
  \else
    \DVert{#1}%
  \fi
}
\newcommand{\inner}[1]{%
  \if@display
    \Dangle*{#1}%
  \else
    \Dangle{#1}%
  \fi
}
\newcommand{\family}[1]{%
  \if@display
    \Dbrace*{#1}%
  \else
    \Dbrace{#1}%
  \fi
}
\newcommand{\Tr}[1]{%
  \if@display
    \operatorname{Tr} \Dparen*{#1}%
  \else
    \operatorname{Tr} \Dparen{#1}%
  \fi
}
\newcommand{\Ran}[1]{%
  \if@display
    \operatorname{Ran} \Dparen*{#1}%
  \else
    \operatorname{Ran} \Dparen{#1}%
  \fi
}
\NewDocumentCommand{\EE}{o m}{%
  \if@display
    \IfNoValueTF{#1}
      {\operatorname{\mathbb{E}} \Dbrack*{#2}}
      {\operatorname{\mathbb{E}}_{#1} \Dbrack*{#2}}%
  \else
    \IfNoValueTF{#1}
      {\operatorname{\mathbb{E}} \Dbrack{#2}}
      {\operatorname{\mathbb{E}}_{#1} \Dbrack{#2}}%
  \fi
}
\NewDocumentCommand{\PP}{o m}{%
  \if@display
    \IfNoValueTF{#1}
      {\operatorname{\mathrm{P}} \Dparen*{#2}}
      {\operatorname{\mathrm{P}}_{#1} \Dparen*{#2}}%
  \else
    \IfNoValueTF{#1}
      {\operatorname{\mathrm{P}} \Dparen{#2}}
      {\operatorname{\mathrm{P}}_{#1} \Dparen{#2}}%
  \fi
}
\makeatother

\makeatletter
\newcommand{\biggg}{\bBigg@\thr@@}
\newcommand{\Biggg}{\bBigg@{3.5}}
\makeatother

\let\oldforall\forall
\renewcommand{\forall}{\mathrel{\oldforall}}
\let\oldexists\exists
\renewcommand{\exists}{\mathrel{\oldexists}}

\DeclareMathOperator*{\argmin}{arg\,min}

\newcommand{\pS}{p^{\textnormal{tr}}}
\newcommand{\pSX}{\pS_{\calX}}
\newcommand{\pSY}{\pS_{\calY}}
\newcommand{\pSYmX}{\pS_{\calY \mid \calX}}
\newcommand{\pSXtY}{\pS_{\calX \times \calY}}
\newcommand{\pT}{p^{\textnormal{te}}}
\newcommand{\pTX}{\pT_{\calX}}
\newcommand{\pTY}{\pT_{\calY}}
\newcommand{\pTYmX}{\pT_{\calY \mid \calX}}
\newcommand{\pTXtY}{\pT_{\calX \times \calY}}
\newcommand{\pXmY}{p_{\calX \mid \calY}}

\newcommand{\KX}{K_{\calX}}
\newcommand{\KY}{K_{\calY}}
\newcommand{\kapX}{\kappa_{\calX}}
\newcommand{\kapY}{\kappa_{\calY}}
\newcommand{\HX}{\calH_{\calX}}
\newcommand{\HY}{\calH_{\calY}}

\newcommand{\IX}{I_{\calX}}
\newcommand{\IY}{I_{\calY}}

\newcommand{\UXY}{U_{\calX \calY}}
\newcommand{\UYX}{U_{\calY \calX}}

\newcommand{\UYY}{U_{\calY \calY}}
\newcommand{\hUXY}{\widehat{U}_{\calX \calY}}
\newcommand{\hUYX}{\widehat{U}_{\calY \calX}}

\newcommand{\hT}{\widehat{T}}
\newcommand{\VXX}{V_{\calX \calX}}
\newcommand{\WXX}{W_{\calX \calX}}
\newcommand{\hWXX}{\widehat{W}_{\calX \calX}}

\newcommand{\etaA}{\eta_{\star}}
\newcommand{\piA}{\varpi_{\star}}
\newcommand{\fTA}{f_{\star}^{\textnormal{te}}}
\newcommand{\psiA}{\psi_{\star}}

\newcommand{\etamu}{\eta_{\mu}}
\newcommand{\hetamu}{\widehat{\eta}_{\mu}}
\newcommand{\tetamu}{\widetilde{\eta}_{\mu}}
\newcommand{\gmu}{g_{\mu}}
\newcommand{\flam}{f_{\lambda}}
\newcommand{\hflam}{\widehat{f}_{\lambda}}
\newcommand{\glam}{g_{\lambda}}

\newcommand{\neta}{n_{\eta}}
\newcommand{\nf}{n_{f}}

\newcommand{\uT}{\bsu^{\textnormal{te}}_{\mathbf{1}}}
\newcommand{\huT}{\widehat{\bsu}^{\textnormal{te}}_{\mathbf{1}}}
\newcommand{\vTy}{\bsv^{\textnormal{te}}_{\boldsymbol{y}}}
\newcommand{\wy}{\bsw_{\boldsymbol{y}}}
\newcommand{\hwy}{\widehat{\bsw}_{\boldsymbol{y}}}

\begin{document}

\title{Learning under Target Shift: Optimal Density Ratio Estimation and Importance-Weighted Regression
}

\author{Ren-Rui Liu\( ^{1} \) \and Zheng-Chu Guo\( ^{1, \ast} \)}

\footnotetext[1]{School of Mathematical Sciences, Zhejiang University, Hangzhou 310058, China}
\renewcommand{\thefootnote}{\fnsymbol{footnote}}
\footnotetext[1]{This work was supported by the National Natural Science Foundation of China (Project No.\ 12271473). The corresponding author is Zheng-Chu Guo. Email address: \href{mailto:guozc@zju.edu.cn}{guozc@zju.edu.cn}.}
\renewcommand{\thefootnote}{\arabic{footnote}}

\date{}

\maketitle

\begin{abstract}
    We study density ratio estimation and importance-weighted regression under target shift with continuous outputs. Under target shift, the conditional distribution of the inputs given the outputs remains invariant across the training and test distributions, while the output marginal distribution may change. Although this problem has been extensively studied for discrete outputs, the continuous setting is substantially less understood: the importance weights are determined by an unknown density ratio function, for which existing estimation methods lack explicit finite-sample convergence rates.
    We propose a spectral regularization method in a reproducing kernel Hilbert space (RKHS) for estimating the continuous density ratio from labeled training samples and unlabeled test inputs. Under a source condition with regularity parameter \(\iota>0\), we establish high-probability finite-sample guarantees and show that the estimator achieves the capacity-independent minimax-optimal RKHS-norm rate \(O(\neta^{-\iota/(2\iota+2)})\). We then incorporate the estimated density ratio into importance-weighted regression and characterize the propagation of density-ratio estimation error to the final predictor. When sufficiently many samples are available for density ratio estimation, the resulting regression estimator attains the minimax-optimal rates of standard kernel regression. These results establish a finite-sample theory for continuous density ratio estimation and importance-weighted learning under target shift.

    \keywords{Learning theory, density ratio estimation, target shift, importance weighting, reproducing kernel Hilbert space, spectral algorithm}
\end{abstract}

\section{Introduction}
\label{sec: introduction}

Classical supervised learning relies on the fundamental assumption that training and test data are drawn from the same probability distribution. Under this assumption, minimizing the empirical risk on training samples yields a predictor that generalizes well to unseen test data with high probability \citep{Vapnik1998StatisticalLT}. In virtually all realistic deployment scenarios, however, this assumption is violated: the training and test distributions differ---a situation broadly termed \emph{distribution shift} \citep{Candela2008DatasetSM, Storkey2008WhenTT}. When this shift occurs, the empirical risk no longer accurately approximates the test risk, and models trained via standard empirical risk minimization can suffer severe degradation in predictive performance.

A principled and widely adopted strategy for correcting distribution shift is \emph{importance weighting} \citep{Shimodaira2000ImprovingPI}. The key idea is to reweight each training example by the \emph{density ratio} (the ratio of the test density to the training density evaluated at that example, formally defined in \eqref{eq: density ratio 1}), thereby constructing an unbiased estimator of the test-expected risk from training data alone. Consequently, learning under distribution shift can be reduced to the auxiliary problem of estimating the unknown density ratio from finite samples.

Among the various forms of distribution shift, two have received particular attention because they induce tractable factorizations of the density ratio. Under \emph{covariate shift}, the conditional distribution of the output given the input is assumed invariant across domains, while the marginal distribution of the inputs may change \citep{Shimodaira2000ImprovingPI, Sugiyama2012DensityRE}. In this case, the density ratio then depends only on the input, and a rich literature has developed methods for its estimation and analyzed the resulting importance-weighted estimators, including sharp theoretical guarantees for kernel methods \citep{Gizewski2022RegularizationUD, Ma2023OptimallyTC, Gogolashvili2023WhenIW, Guo2025OnlineLA, Fan2025SpectralAU}. Under \emph{target shift}, the roles of the input and output are reversed: the conditional distribution of the input given the output remains invariant, while the marginal distribution of the output changes between the training and test domains \citep{Saerens2002AdjustingOC, Zhang2013DomainAU}. This setting arises naturally when the output can be viewed as a cause of the observed input features. For example, in medical diagnosis, a disease (output) causes symptoms (inputs); the prevalence of the disease may differ across hospitals, yet the distribution of symptoms given the disease remains the same. Similarly, in econometric forecasting, a policy intervention (output) drives economic indicators (inputs); the policy mix may shift over time, while the conditional distribution of the indicators given the policy stays invariant. Under target shift, the density ratio therefore reduces to a function of the output alone, namely the ratio of the test to the training marginal density of the outputs.

The vast majority of existing work on target shift has focused on the \emph{discrete} case, where the importance weight reduces to a finite-dimensional vector and a mature literature provides efficient algorithms with strong theoretical guarantees \citep{Saerens2002AdjustingOC, Lipton2018DetectingCL, Azizzadenesheli2019RegularizedLD, Alexandari2020MaximumLB, Garg2020UnifiedVL}. In the \emph{continuous} setting, by contrast, the importance weight becomes an infinite-dimensional function whose estimation typically involves solving an ill-posed integral equation. Despite its practical importance, density ratio estimation under continuous target shift remains comparatively underexplored. A few pioneering algorithms have been proposed \citep{Zhang2013DomainAU, Nguyen2016ContinuousTS, Kim2024ReTaSANF}, but a finite-sample theory with explicit, optimized convergence rates for the estimated density ratio is still lacking. In particular, \citet{Kim2024ReTaSANF} established consistency, but the resulting rates depend on bandwidth and regularization parameters that are not optimized with respect to the sample size; More recently, \citet{Gogolashvili2026ImportanceWC} derived minimax-optimal rates for importance-weighted regression under continuous target shift, but assumed that the true density ratio is known. Thus, establishing explicit finite-sample convergence rates for estimating the density ratio itself remains an important open problem.

\noindent \textbf{Contributions.} This work develops a finite-sample learning theory for continuous density ratio estimation under target shift and investigates how density ratio estimation error propagates to downstream importance-weighted regression. Our main contributions are summarized as follows.

\begin{itemize}
    \item \textbf{A spectral framework for continuous density ratio estimation.}
          We propose a spectral algorithm in an RKHS for estimating the density ratio between the output marginal distributions using labeled training samples and unlabeled test inputs. By casting density ratio estimation as a regularized operator equation, the proposed framework accommodates a broad class of spectral regularization methods.

    \item \textbf{Minimax-optimal finite-sample guarantees for density ratio estimation.}
          We establish explicit finite-sample convergence bounds for the proposed density ratio estimator and show that, under suitable source conditions, it achieves the minimax-optimal convergence rate in the RKHS norm. To the best of our knowledge, this provides the first such finite-sample guarantee for continuous density ratio estimation under target shift, going beyond existing results that either assume the density ratio to be known or establish only asymptotic consistency.

    \item \textbf{Learning theory for importance-weighted regression with an estimated density ratio.}
          We incorporate the estimated density ratio into an importance-weighted spectral algorithm for the downstream regression problem and explicitly characterize the propagation of density ratio estimation error into the final prediction error. Our analysis identifies two statistical regimes: when sufficiently many samples are allocated to density ratio estimation, the regression estimator retains the minimax-optimal learning rates of standard kernel regression; when the available sample size is insufficient, density ratio estimation error becomes dominant and leads to slower convergence rates.

    \item \textbf{A quantitative characterization of the interaction between the two learning stages.}
          Our analysis explicitly characterizes how the sample size for density ratio estimation, the regularity of both the density ratio and the target function, and the sample size for downstream regression jointly determine the overall convergence rate. These results provide a rigorous theoretical foundation for importance weighting under continuous target shift and offer quantitative guidance on the sample requirements for density ratio estimation relative to those for downstream regression.
\end{itemize}
\noindent \textbf{Paper organization.} The remainder of this paper is organized as follows: Section \ref{section: problem setting and main results} introduces the problem setting, presents the main assumptions, and states the main theoretical results; \secref{sec: discussion} provides a literature review and comparative analysis with existing works; \secref{sec: proof} contains the proofs of the main theorems.



\section{Problem Setting and Main Results}\label{section: problem setting and main results}
We begin by considering the regression setting considered in this paper. Let \( \calX \subseteq \bbR^{d} \) and \( \calY \subseteq \bbR \) be compact input and output spaces, respectively, and consider the squared loss. Our goal is to learn a predictor from training data \( \family{ (\bsx_{j}, y_{j}) }_{j=1}^{\nf} \subseteq \calX \times \calY \) that achieves good generalization performance under a test distribution \( \pTXtY(\bsx,y)\), when the training data are generated from a potentially different distribution \( \pSXtY(\bsx,y) \). Both the training and test distributions are unknown and can only be accessed through finite samples.

For a predictor \( f \), its generalization performance is measured by the expected risk under the test distribution:
\begin{equation}
    \label{eq: risk expected}
    \EE[(\bsx, y) \sim \pTXtY]{ (y - f(\bsx))^{2} }.
\end{equation}
The minimizer of the expected risk is the regression function \( \fTA \), defined as
\[
    \fTA(\bsx) = \EE[y \sim \pTYmX]{ y \mid \bsx },
\]
where \( \pTYmX(y \mid \bsx) \) denotes the test-conditional distribution of \( y \) given \( \bsx \). Hence, learning a predictor with small test risk amounts to accurately estimating  \( \fTA \). Since the test distribution \( \pTXtY \) is unknown, the expected risk \eqref{eq: risk expected} cannot be computed directly. Instead, we approximate it by the empirical risk over the training sample:
\begin{equation}
    \label{eq: risk empirical 1}
    \frac{1}{\nf} \sum_{j=1}^{\nf} (y_{j} - f(\bsx_{j}))^{2}.
\end{equation}
In the standard learning setting, where the training and test distributions coincide, i.e., (\( \pSXtY = \pTXtY \)), the empirical risk \eqref{eq: risk empirical 1} is an unbiased estimator of the expected risk, and minimizing it often provides a consistent estimator of \( \fTA \) under suitable conditions \citep{Vapnik1998StatisticalLT}. Under \emph{distribution shift}, however, \( \pSXtY \ne \pTXtY \) and \eqref{eq: risk empirical 1} becomes a biased estimator. Consequently, minimizing the unweighted empirical risk may lead to a predictor that generalizes poorly under the test distribution.

To correct this bias, we can reweight each training pair by the \emph{density ratio} between the test and training distributions. Define the density ratio \( \etaA \) as
\begin{equation}
    \label{eq: density ratio 1}
    \etaA(\bsx, y) = \frac{\pTXtY(\bsx, y)}{\pSXtY(\bsx, y)},
\end{equation}
and form the weighted empirical risk
\begin{equation}
    \label{eq: risk empirical 2}
    \frac{1}{\nf} \sum_{j=1}^{\nf} \etaA(\bsx_{j}, y_{j}) (y_{j} - f(\bsx_{j}))^{2},
\end{equation}
then \eqref{eq: risk empirical 2} is again an unbiased estimator of the expected risk \eqref{eq: risk expected}. This approach is known as \emph{importance weighting} \citep{Shimodaira2000ImprovingPI}. In this paper, we focus exclusively on \emph{target shift}, where the marginal distributions of \( y \) differ between training and test, while the conditional distribution of \( \bsx \) given \( y \) remains invariant. That is,
\[
    \pSXtY(\bsx, y) = \pXmY(\bsx \mid y) \, \pSY(y),
    \quad \pTXtY(\bsx, y) = \pXmY(\bsx \mid y) \, \pTY(y),
\]
with \( \pSY \ne \pTY \) and \( \pXmY \) identical across both distributions. Under target shift, the density ratio simplifies to a function of \( y \) only:
\begin{equation}
    \label{eq: density ratio 2}
    \etaA(\bsx, y) = \etaA(y) = \frac{\pTY(y)}{\pSY(y)}.
\end{equation}

In practice, the density ratio \( \etaA \) is unknown and need to be estimated from data.
In classification, where \( \calY \) is discrete, \( \etaA \) reduces to a finite-dimensional vector and can, in principle, be estimated directly from labeled samples from the training and test distributions. In regression, however, \( \etaA \) is an unknown function over a continuous output space. Direct estimation of \( \pTY/\pSY \) would therefore require sufficiently many labeled test samples, which are often unavailable or prohibitively expensive to obtain. To avoid this requirement, we assume access only to labeled training samples \( \family{ (\bsx_{i}, y_{i}) }_{i=1}^{\neta} \) drawn from \( \pSXtY \), and unlabeled data \( \family{ \bsx'_{i} }_{i=1}^{\neta} \) drawn from the test marginal distribution \( \pTX \). Our goal is to exploit the dependence between the inputs and outputs, together with the target-shift assumption, to estimate the continuous density ratio \( \etaA \) without requiring labeled test data.

In the following sections, we first develop a method for estimating the continuous density ratio \( \etaA \), and then incorporate the resulting estimate into the importance-weighting framework to learn the regression function \( \fTA \).

\subsection{Our Estimation Methods}

\label{subsec: method}

We estimate the density ratio \( \etaA \) from a labeled sample \( \family{ (\bsx_{i}, y_{i}) }_{i=1}^{\neta} \sim \pSXtY \) and an unlabeled sample \( \family{ \bsx'_{i} }_{i=1}^{\neta} \sim \pTX \). To construct the estimator, we first introduce the underlying function spaces. Let \( \KX \colon \calX \times \calX \to \bbR_{+} \) and \( \KY \colon \calY \times \calY \to \bbR_{+} \) be symmetric, positive-definite, and continuous kernels. These kernels induce two reproducing kernel Hilbert spaces (RKHSs), denoted by \( \HX \) and \( \HY \), whose elements satisfy the reproducing properties:
\begin{alignat*}{3}
          & \inner{ f, \KX(\cdot, \bsx) }_{\HX} = f(\bsx),
    \quad & \forall f \in \HX,
    \quad & \forall \bsx \in \calX;                        \\
          & \inner{ h, \KY(\cdot, y) }_{\HY} = h(y),
    \quad & \forall h \in \HY,
    \quad & \forall y \in \calY.
\end{alignat*}
Next, we assume that the joint distributions on \( \calX \times \calY \) factorize as follows:
\begin{alignat*}{2}
          & \pSXtY(\bsx, y) = \pXmY(\bsx \mid y) \, \pSY(y),
    \quad & \pTXtY(\bsx, y) = \pXmY(\bsx \mid y) \, \pTY(y),     \\
          & \pSXtY(\bsx, y) = \pSYmX(y \mid \bsx) \, \pSX(\bsx),
    \quad & \pTXtY(\bsx, y) = \pTYmX(y \mid \bsx) \, \pTX(\bsx).
\end{alignat*}
The first factorization reflects the target shift assumption, under which the conditional distribution \( \pXmY \) remains invariant; the second is the standard chain-rule factorization, and \( \pSYmX \) does not necessarily equal \( \pTYmX \). Using these factorizations, we derive a key relation for the density ratio \( \etaA(y) = \pTY(y) / \pSY(y) \):
\begin{align*}
    \frac{\pTX(\bsx)}{\pSX(\bsx)}
     & = \int_{\calY} \frac{\pTXtY(\bsx, y)}{\pSX(\bsx)} \, \dd y
    = \int_{\calY} \frac{\pXmY(\bsx \mid y) \, \pTY(y)}{\pSX(\bsx)} \, \dd y                                  \\
     & = \int_{\calY} \frac{\pTY(y)}{\pSY(y)} \cdot \frac{\pXmY(\bsx \mid y) \, \pSY(y)}{\pSX(\bsx)} \, \dd y
    = \int_{\calY} \etaA(y) \cdot \pSYmX(y \mid \bsx) \, \dd y.
\end{align*}
Multiplying both sides by \( \KX(\cdot, \bsx) \, \pSX(\bsx) \) and integrating over \( \calX \) yields
\begin{align*}
    \int_{\calX} \KX(\cdot, \bsx) \, \pTX(\bsx) \, \dd \bsx
     & = \int_{\calX} \frac{\pTX(\bsx)}{\pSX(\bsx)} \, \KX(\cdot, \bsx) \, \pSX(\bsx) \, \dd \bsx                                   \\
     & = \int_{\calX} \LE( \int_{\calY} \etaA(y) \cdot \pSYmX(y \mid \bsx) \, \dd y \RI) \KX(\cdot, \bsx) \, \pSX(\bsx) \, \dd \bsx \\
     & = \int_{\calX \times \calY} \etaA(y) \, \KX(\cdot, \bsx) \, \pSXtY(\bsx, y) \, \dd \bsx \dd y.
\end{align*}
Equivalently, we obtain the following identity:
\begin{equation}
    \label{eq: key relation 0}
    \EE[(\bsx, y) \sim \pSXtY]{ \etaA(y) \, \KX(\cdot, \bsx) } = \EE[\bsx' \sim \pTX]{ \KX(\cdot, \bsx') }.
\end{equation}
Let \( \uT \) be the kernel mean embedding of the test marginal distribution \( \pTX \), i.e.,
\[
    \uT = \EE[\bsx' \sim \pTX]{ \KX(\cdot, \bsx') },
\]
and introduce the following cross-covariance operators~\footnote{
    For \( f_{0} \in \mathscr{H}_{1} \) and \( h_{0} \in \mathscr{H}_{2} \), the tensor product \( f_{0} \otimes h_{0} \) defines a rank-\( 1 \) operator as
    \[
        f_{0} \otimes h_{0} \colon \quad \mathscr{H}_{2} \to \mathscr{H}_{1}, \quad h \mapsto \inner{ h_{0}, h }_{\mathscr{H}_{2}} \, f_{0}.
    \]
}:
\[
    \UYX = \EE[(\bsx, y) \sim \pSXtY]{ \KY(\cdot, y) \otimes \KX(\cdot, \bsx) },
    \quad \UXY = \EE[(\bsx, y) \sim \pSXtY]{ \KX(\cdot, \bsx) \otimes \KY(\cdot, y) }.
\]
Assuming that \( \etaA \in \HY \), the identity \eqref{eq: key relation 0} then simplifies to the operator equation
\begin{equation}
    \label{eq: key relation 1}
    \UXY \, \etaA = \uT.
\end{equation}

We exploit \eqref{eq: key relation 1} to construct an estimator for \( \etaA \). The empirical counterparts of the involved quantities are
\begin{equation}
    \label{equation: hatuT}
    \huT = \frac{1}{\neta} \sum_{i=1}^{\neta} \KX(\cdot, \bsx'_{i}),
\end{equation}
and
\begin{equation}
    \label{equation: hatUYX and hatUXY}
    \hUYX = \frac{1}{\neta} \sum_{i=1}^{\neta} \KY(\cdot, y_{i}) \otimes \KX(\cdot, \bsx_{i}),
    \quad \hUXY = \frac{1}{\neta} \sum_{i=1}^{\neta} \KX(\cdot, \bsx_{i}) \otimes \KY(\cdot, y_{i}).
\end{equation}
Solving the regularized least-squares problem
\[
    \argmin_{\eta \in \HY} \norm{ \hUXY \, \eta - \huT }_{\HX}^{2} + \mu \norm{ \eta }_{\HY}^{2},
\]
yields the kernel ridge regression estimator \( \hetamu^{\textnormal{KRR}} \) of \( \etaA \):
\[
    \hetamu^{\textnormal{KRR}} = (\hT + \mu I_{\calY})^{-1} \, \hUYX \, \huT,
\]
where \( \mu > 0 \) is the regularization parameter and \( \hT = \hUYX \, \hUXY \). The idea of kernel ridge regression extends naturally to a broader class of regularization methods, known as \emph{spectral algorithms} \citep{Vito2005LearningFE, Gerfo2008SpectralAS, Bauer2007RegularizationAL}. A spectral algorithm modifies the spectrum of the underlying operator by applying a scalar function to each eigenvalue, shrinking the contributions associated with small eigenvalues to control the effective complexity of the hypothesis space. This function is termed the \emph{filter function} and completely characterizes the regularization strategy.
\begin{definition}[Filter functions]
    \label{def: filter}
    A family of functions \( \gmu \colon [0, \kappa^{2}] \to [0, \infty) \), parameterized by \( \mu > 0 \), constitutes filter functions if:
    \begin{itemize}
        \item There exists \( E \ge 0 \) such that for all \( c \in [0, 1] \):
              \begin{equation}
                  \label{eq: filter E}
                  \sup_{t \in [0, \kappa^{2}]} t^{c} \gmu(t) \le E \cdot \mu^{c - 1}.
              \end{equation}
        \item There exist \( \tau \ge 1 \) and \( F \ge 0 \) such that for all \( c \in [0, \tau] \):
              \begin{equation}
                  \label{eq: filter F}
                  \sup_{t \in [0, \kappa^{2}]} t^{c} |1 - t \gmu(t)| \le F \cdot \mu^{c}.
              \end{equation}
    \end{itemize}
\end{definition}
Condition \eqref{eq: filter E} ensures that the regularized inverse remains bounded, guaranteeing numerical stability. Condition \eqref{eq: filter F} controls the approximation error by requiring the residual \( |1 - t \gmu(t)| \) to decay at a rate governed by \( \mu \). The parameter \( \tau \), known as the \emph{qualification} of the regularization method, determines the maximum degree of source smoothness that the algorithm can effectively handle. This framework, originally developed for solving ill-posed linear inverse problems \citep{Engl2015RegularizationIP}, has been adapted to the learning theory \citep{Guo2017LearningTD, Fan2025SpectralAU, Liu2025SpectralAM, Liu2026UnboundedDR} and encompasses a variety of regularization methods, including:
\begin{itemize}
    \item Kernel ridge regression:
          \( \gmu^{\, \text{KRR}}(t) = (t + \mu)^{-1} \), with qualification \( \tau = 1 \) and constants \( E = F = 1 \); here \( \mu \) serves as the regularization parameter.

    \item Early-stopped gradient flow:
          \( \gmu^{\, \text{GF}}(t) = t^{-1}(1 - \ee^{-t/\mu}) \), which achieves arbitrary qualification \( \tau \ge 1 \) with \( E = 1 \), \( F = (\tau / \ee)^{\tau} \); the stopping time corresponds to \( 1 / \mu \).

    \item Spectral cutoff:
          \( \gmu^{\, \text{CUT}}(t) = t^{-1} \ind{t \ge \mu} \), with arbitrary \( \tau \ge 1 \) and \( E = F = 1 \); the cutoff threshold is \( \mu \).
\end{itemize}
Since the input and output spaces are compact, the kernels are bounded, and we set
\[
    \sup_{\bsx \in \calX} \KX(\bsx, \bsx) \le \kapX^{2},
    \quad \sup_{y \in \calY} \KY(y, y) \le \kapY^{2}.
\]
Consequently, the operator norms
\[
    \norm{ \UYX }_{\HX \to \HY},
    \quad \norm{ \hUYX }_{\HX \to \HY},
    \quad \norm{ \UXY }_{\HY \to \HX},
    \quad \norm{ \hUXY }_{\HY \to \HX}
\]
are all bounded above by \( \kapX \kapY \). Therefore \( \norm{ \hT }_{\HY \to \HY} = \norm{ \hUXY \, \hUYX }_{\HY \to \HY} \le (\kapX \kapY)^{2} \). Hence, \( \kappa \) in \defref{def: filter} can be taken as \( \kappa = \kapX \kapY \), and the spectral algorithm yields the following estimator of \( \etaA \):
\begin{equation}
    \label{eq: hetamu}
    \hetamu = \gmu (\hT) \, \hUYX \, \huT.
\end{equation}
Finally, since the true density ratio is nonnegative, we refine our estimate by retaining only its positive part:
\[
    \tetamu = \max \family{ \hetamu, 0 }.
\]
Our final density ratio estimator is therefore \( \tetamu \).

With the density ratio estimator established, we now estimate the regression function \( \fTA \), using an independent sample \( \family{ (\bsx_{j}, y_{j}) }_{j=1}^{\nf} \sim \pSXtY \). Recalling that \( \fTA(\bsx) = \EE[y \sim \pTYmX]{ y \mid \bsx } \), we have the key relation:
\begin{align*}
    \int_{\calX} \fTA(\bsx) \, \KX(\cdot, \bsx) \, \pTX(\bsx) \, \dd \bsx
     & = \int_{\calX} \LE( \int_{\calY} y \, \pTYmX(y \mid \bsx) \, \dd y \RI) \KX(\cdot, \bsx) \, \pTX(\bsx) \, \dd \bsx \\
     & = \int_{\calX \times \calY} y \, \KX(\cdot, \bsx) \, \pTX(\bsx, y) \, \dd \bsx \dd y.
\end{align*}
That is,
\[
    \EE[\bsx \sim \pTX]{ \fTA(\bsx) \, \KX(\cdot, \bsx) } = \EE[(\bsx, y) \sim \pTXtY]{ y \, \KX(\cdot, \bsx) }.
\]
To simplify notation, define the response kernel embedding on the test domain as \( \vTy \):
\[
    \vTy = \EE[(\bsx, y) \sim \pTXtY]{ y \, \KX(\cdot, \bsx) },
\]
and the covariance operator on the test domain as \( \VXX \):
\[
    \VXX = \EE[\bsx \sim \pTX]{ \KX(\cdot, \bsx) \otimes \KX(\cdot, \bsx) }.
\]
Then, assuming \( \fTA \in \HX \), we obtain
\begin{equation}
    \label{eq: key relation 2}
    \VXX \, \fTA = \vTy.
\end{equation}

To estimate \( \fTA \) using \eqref{eq: key relation 2}, we introduce the auxiliary operator and quantity
\[
    \WXX = \EE[(\bsx, y) \sim \pSXtY]{ \tetamu(y) \, \KX(\cdot, \bsx) \otimes \KX(\cdot, \bsx) },
    \quad \wy = \EE[(\bsx, y) \sim \pSXtY]{ \tetamu(y) y \, \KX(\cdot, \bsx) }.
\]
These are the importance-weighted versions of \( \VXX \) and \( \vTy \). Their empirical counterparts are
\[
    \hWXX = \frac{1}{\nf} \sum_{j=1}^{\nf} \tetamu(y_{j}) \, \KX(\cdot, \bsx_{j}) \otimes \KX(\cdot, \bsx_{j}),
    \quad \hwy = \frac{1}{\nf} \sum_{j=1}^{\nf} \tetamu(y_{j}) y_{j} \, \KX(\cdot, \bsx_{j}).
\]
Applying a spectral algorithm then gives an estimator of \( \fTA \):
\begin{equation}
    \label{eq: hflam}
    \hflam = \glam(\hWXX) \, \hwy,
\end{equation}
where \( \glam \) is a filter function from \defref{def: filter} with regularization parameter \( \lambda > 0 \). Note that as \( \neta \) increases, with a suitable choice of \( \mu \), the estimate \( \hetamu \) (and hence \( \tetamu = \max \family{ \hetamu, 0 } \)) converges to \( \etaA \) with high probability; consequently, \( \hWXX \) converges to \( \VXX \) with high probability. Therefore, without loss of generality, we may set \( \kappa \) in \defref{def: filter} as \( \kapX^{2} \) here, since \( \norm{ \VXX }_{\HX \to \HX} \le \kapX^{2} \).


\subsection{Main Results}
\label{sec: result}

In this subsection, we derive high-probability convergence guarantees for the density ratio estimator \( \hetamu \) and the subsequent regression function estimator \( \hflam \).

We first analyze the convergence of \( \hetamu \) to the true density ratio \( \etaA \). Recall that \( \hetamu = \gmu (\hT) \, \hUYX \, \huT \) is defined by the spectral regularization scheme \eqref{eq: hetamu}, using samples \( \family{ (\bsx_{i}, y_{i}) }_{i=1}^{\neta} \sim \pSXtY \) and \( \family{ \bsx'_{i} }_{i=1}^{\neta} \sim \pTX \).
Here, \( \gmu \) is the filter function introduced in \defref{def: filter} with regularization parameter \( \mu > 0 \), and \( \hT = \hUYX \, \hUXY \). The quantities  \(\huT, \hUYX \) and \(\hUXY \) are defined by \eqref{equation: hatuT} and \eqref{equation: hatUYX and hatUXY} respectively.


To obtain a quantitative convergence rate, we impose a source (or regularity) condition on \( \etaA \) relative to the operator underlying the estimation problem. Specifically, we assume that \( \etaA \) lies in the range of a fractional power of \( T = \UYX \, \UXY \).

\begin{assumption}[Source condition of \( \etaA \)]
    \label{asp: source etaA}
    Let \( \tau \) be the qualification parameter from \defref{def: filter} and let \( \etaA = \pTY / \pSY \) denote the density ratio. There exist \( \iota \in (0, \tau] \) and \( \piA \in \HY \) such that
    \[
        \etaA = T^{\iota} \, \piA.
    \]
\end{assumption}

In \aspref{asp: source etaA}, the exponent \( \iota \) quantifies the regularity of \( \etaA \) with respect to the spectral decomposition of \( T \), with larger values of \( \iota \) corresponding to stronger regularity. The restriction \( \iota \leq \tau \) ensures that the qualification of the filter \( \gmu \) is sufficient to exploit the prescribed source condition without saturation. Source conditions of this form are standard in the analysis of kernel-based learning and spectral regularization methods \citep{Smale2003EstimatingAE, Cucker2007LearningTA, Caponnetto2007OptimalRR}.

\begin{remark}
    \aspref{asp: source etaA} not only quantifies the regularity of \( \etaA \), but also ensures its identifiability from the data. The fundamental relationship \eqref{eq: key relation 1},
    \[
        \UXY \, \etaA = \uT,
    \]
    generally admits infinitely many solutions because \( \UXY \) may possess a nontrivial null space. However, the source condition \( \etaA = T^{\iota} \, \piA \) with \( T = \UYX \, \UXY \) forces \( \etaA \) to be orthogonal to that null space: for any \( h \in \HY \) satisfying \( \UXY \, h = 0 \), we have \( T \, h = 0 \), hence \( T^{\iota} \, h = 0 \), and therefore
    \[
        \inner{ \etaA, h }_{\HY} = \inner{ T^{\iota} \, \piA, h }_{\HY} = \inner{ \piA, T^{\iota} \, h }_{\HY} = 0.
    \]
    Thus \( \etaA \) coincides exactly with the unique minimal-norm solution of \( \UXY \, \eta = \uT \). This property guarantees that the true density ratio can in principle be recovered by regularization algorithms that promote small norms, and it underpins the convergence analysis that follows.
\end{remark}

In kernel-based learning theory, it is customary to study both the RKHS-norm and the \( \calL^{2} \)-norm. These two norms are linked through the covariance operator
\[
    \UYY = \EE[(\bsx, y) \sim \pSXtY]{ \KY(\cdot, y) \otimes \KY(\cdot, y) }.
\]
Indeed, for every \( h \in \HY \) the isometric property \( \norm{ h }_{\calL^{2}(\calY, \pSY)} = \norm{ \UYY^{1 / 2} \, h }_{\HY} \) holds. A standard route to \( \calL^{2} \)-norm bounds is therefore to write
\begin{equation}
    \label{eq: isometric isomorphism}
    \norm{ \hetamu - \etaA }_{\calL^{2}(\calY, \pSY)} = \norm{ \UYY^{1 / 2} \, (\hetamu - \etaA) }_{\HY},
\end{equation}
and then control the RKHS-norm on the right-hand side. In the present setting, the estimator \( \hetamu \) is built from the empirical operator \( \hT = \hUYX \, \hUXY \); by concentration, its error naturally decomposes along the spectral calculus of \( T = \UYX \, \UXY \). To translate the right-hand side of \eqref{eq: isometric isomorphism} into a form that can be handled via such spectral estimates---for instance, \( \norm{ T^{\alpha / 2} \, (\hetamu - \etaA) }_{\HY} \)---one would need an inequality of the form
\begin{equation}
    \label{eq: reverse}
    \UYY \preceq c T^{\alpha},
    \quad c > 0, \quad \alpha > 0,
\end{equation}
relating the two positive self-adjoint operators on \( \HY \). (The notation \( \preceq \) means that \( c T^{\alpha} - \UYY \) is positive semidefinite.)

However, while the elementary bound \( T \preceq \kapX^{2} \UYY \) always holds~\footnote{
    To prove this bound, it suffices to show \( \norm{ T^{1/2} \, h }_{\HY}^{2} \le \kapX^{2} \cdot \norm{ \UYY^{1/2} \, h }_{\HY}^{2} \) holds for all \( h \in \HY \). Indeed, we have
    \begin{align*}
        \norm{ T^{1 / 2} \, h }_{\HY}^{2}
         & = \norm{ \UXY \, h }_{\HX}^{2}
        = \norm{ \EE[(\bsx, y) \sim \pSXtY]{ h(y) \, \KX(\cdot, \bsx) } }_{\HX}^{2}
        \le \EE[(\bsx, y) \sim \pSXtY]{ \norm{ h(y) \, \KX(\cdot, \bsx) }_{\HX}^{2} } \\
         & = \EE[(\bsx, y) \sim \pSXtY]{ h^{2}(y) \, \KX(\bsx, \bsx) }
        \le \kapX^{2} \EE[(\bsx, y) \sim \pSXtY]{ h^{2}(y) }
        = \kapX^{2} \cdot \norm{ h }_{\calL^{2}(\calY, \pSY)}^{2}                     \\
         & = \kapX^{2} \cdot \norm{ \UYY^{1/2} \, h }_{\HY}^{2}.
    \end{align*}
    The statement is thus proved.
}, the reverse inequality does \emph{not} hold in general. The operator \( T \) captures only the variability in \( \calY \) that remains distinguishable after conditioning on \( \calX \). If two distinct targets \( y_{1} \ne y_{2} \) give rise to nearly indistinguishable conditional distributions \( \pXmY(\cdot \mid y_{1}) \approx \pXmY(\cdot \mid y_{2}) \), then \( \UXY \) maps the corresponding directions in \( \HY \) to almost the same point in \( \HX \), so that \( T \) strongly suppresses them. In contrast, \( \UYY \) weights every direction according to the marginal \( \pSY \), irrespective of its discriminability from \( \calX \). Consequently, the spectrum of \( T \) may be markedly smaller than that of \( \UYY \), and a reverse domination of the form \( \UYY \preceq c T^{\alpha} \) cannot be guaranteed without strong additional assumptions.

Without a reverse bound of the type \eqref{eq: reverse}, decay rates for \( \norm{ T^{\alpha / 2} \, (\hetamu - \etaA) }_{\HY} \) cannot be converted into meaningful rates for \( \norm{ \UYY^{1 / 2} \, (\hetamu - \etaA) }_{\HY} = \norm{ \hetamu - \etaA }_{\calL^{2}(\calY, \pSY)} \). For this reason, the theorem that follows provides convergence guarantees only in the \( \HY \)-norm.

\begin{theorem}
    \label{thm: density ratio}
    Suppose that \aspref{asp: source etaA} holds with \( \iota \in (0, \tau] \) and set the regularization parameter \( \mu = \neta^{-1 / (2 \iota + 2)} \). Then, for any \( \delta \in (0, 1) \), with probability at least \( 1 - \delta \), if the sample size \( \neta \) satisfies
    \begin{equation}
        \label{eq: sample size condition 1}
        \begin{aligned}
            \neta \ge \LE( 40 (\kapX \kapY)^{2} \log \frac{4}{\delta} \RI)^{\frac{2 \iota + 2}{\iota}},
        \end{aligned}
    \end{equation}
    the density ratio estimator \( \hetamu \) defined in \eqref{eq: hetamu} obeys
    \[
        \norm{ \etaA - \hetamu }_{\HY} \le \varDelta_{\eta} \cdot \neta^{-\frac{\iota}{2 \iota + 2}} \log \frac{4}{\delta},
    \]
    where \( \varDelta_{\eta} \) is a constant independent of \( \neta \) and \( \delta \), specified in \eqref{eq: Delta_eta}.
\end{theorem}

\thmref{thm: density ratio} shows that the density ratio estimator \( \hetamu \) attains the convergence rate \( O(\neta^{-\iota / (2 \iota + 2)}) \) in the RKHS-norm. According to \citet{Caponnetto2007OptimalRR}, this rate matches the capacity-independent minimax optimal rate under a source condition with parameter \( \iota > 0 \).

Having obtained an estimator \( \hetamu \) of the density ratio, we next incorporate it into the importance-weighting framework to estimate the target regression function
\( \fTA \). To ensure nonnegativity of the estimated importance weights, we define the truncated estimator: \( \tetamu = \max\{ \hetamu, 0 \} \). The regression function estimator is then defined as in \eqref{eq: hflam}:
\[
    \hflam = \glam (\hWXX) \, \hwy,
\]
where \( \glam \) is the filter function from \defref{def: filter} with regularization parameter \( \lambda > 0 \), and
\[
    \hWXX = \frac{1}{\nf} \sum_{j=1}^{\nf} \tetamu(y_{j}) \, \KX(\cdot, \bsx_{j}) \otimes \KX(\cdot, \bsx_{j}),
    \quad \hwy = \frac{1}{\nf} \sum_{j=1}^{\nf} \tetamu(y_{j}) y_{j} \, \KX(\cdot, \bsx_{j}).
\]
To analyze the convergence of \( \hflam \), we impose a source condition that characterizes the regularity of the regression function \( \fTA \), analogously to \aspref{asp: source etaA}.

\begin{assumption}[Source condition of \( \fTA \)]
    \label{asp: source fTa}
    Let \( \tau \) be the qualification parameter from \defref{def: filter}. There exists \( r \in (0, \tau - 1 / 2] \) such that
    \[
        \fTA = \VXX^{r} \, \psiA
    \]
    holds for some \( \psiA \in \HX \).
\end{assumption}

We are now ready to state the convergence guarantees for the regression estimator.

\begin{theorem}
    \label{thm: regression function}
    Let \( \delta \in (0,1) \). Assume that the conditions of \thmref{thm: density ratio} hold with \( \iota \in (0,\tau] \),
    and that \aspref{asp: source fTa} holds with \( r \in (0, \tau - 1 / 2] \). In addition to \eqref{eq: sample size condition 1}, assume that the sample sizes \( \neta,\nf \) and the regularization parameter \( \lambda \) satisfy
    \begin{equation}
        \label{eq: sample size condition 2}
        \begin{cases}
            \displaystyle \lambda \neta^{\frac{\iota}{2 \iota + 2}} \ge 4 \kapX^{2} \kapY \varDelta_{\eta} \log \frac{8}{\delta}, \\[1em]
            \displaystyle \lambda \nf^{1 / 2} \ge 40 M_{\eta} \kapX^{2} \log \frac{8}{\delta}.
        \end{cases}
    \end{equation}
    Then, with probability at least \( 1 - \delta \), the estimator \( \hflam \) achieves the following error bounds:
    \begin{enumerate}
        \item If \( \neta^{2 \iota / (2 \iota + 2)} < \nf \), choosing
              \[
                  \lambda = \neta^{- \frac{\iota}{2 \iota + 2} \cdot \frac{1}{r + 1}},
              \]
              yields
              \[
                  \norm{ \fTA - \hflam }_{\calL^{2}(\calX, \pTX)}
                  \le \varDelta_{f} \cdot \neta^{- \frac{\iota}{2 \iota + 2} \cdot \frac{r + 1 / 2}{r + 1}} \log \frac{8}{\delta},
                  \quad \norm{ \fTA - \hflam }_{\HX}
                  \le \varDelta_{f} \cdot \neta^{- \frac{\iota}{2 \iota + 2} \cdot \frac{r}{r + 1}} \log \frac{8}{\delta}.
              \]

        \item If \( \neta^{2 \iota / (2 \iota + 2)} \ge \nf \), choosing
              \[
                  \lambda = \nf^{-\frac{r}{2 r + 2}},
              \]
              yields
              \[
                  \norm{ \fTA - \hflam }_{\calL^{2}(\calX, \pTX)}
                  \le \varDelta_{f} \cdot \nf^{-\frac{r + 1 / 2}{2 r + 2}} \log \frac{8}{\delta},
                  \quad \norm{ \fTA - \hflam }_{\HX}
                  \le \varDelta_{f} \cdot \nf^{-\frac{r}{2 r + 2}} \log \frac{8}{\delta}.
              \]
    \end{enumerate}
    Here, \( \varDelta_{f} \) is a constant given in \eqref{eq: Delta_f} that is independent of \( \neta \), \( \nf \), and \( \delta \).
\end{theorem}

According to \citet{Caponnetto2007OptimalRR}, under a source condition with parameter \( r > 0 \), the capcity-independent minimax optimal convergence rates  are \( O(\nf^{-(r + 1 / 2) / (2 r + 2)}) \) in the \( \calL^{2} \)-norm and \( O(\nf^{-r / (2 r + 2)}) \) in the RKHS-norm. Hence, \thmref{thm: regression function} shows that when \( \neta^{2 \iota / (2 \iota + 2)} \ge \nf \), the estimator \( \hflam \) attains these minimax optimal rates. In contrast, when \( \neta^{2 \iota / (2 \iota + 2)} < \nf \), only suboptimal rates are obtained. Because \( 2 \iota / (2 \iota + 2) < 1 \) for all \( \iota > 0 \), the favorable regime requires the sample size \( \neta \) for density ratio estimation to scale polynomially with the sample size \( \nf \) for regression estimation. This condition is somewhat stringent, as both steps rely on costly labeled sample pairs.

\section{Related Work and Discussions}
\label{sec: discussion}

In this section, we review the literature on density ratio estimation and target shift, discuss the implications of our main results, and conclude with limitations of our analysis and several directions for future research.

\subsection{Density Ratio Estimation}

Estimating the density ratio between two probability distributions is a fundamental problem in many areas of machine learning \citep{Sugiyama2012DensityRE}. A naive approach estimates the numerator and denominator densities separately and then takes their ratio, typically using parametric models or nonparametric techniques such as Kernel Density Estimation (KDE) \citep{Sheather1991ReliableDB}. However, this indirect method suffers from the curse of dimensionality and amplifies estimation errors, particularly in regions where the denominator density is near zero. To overcome these limitations, the prevailing paradigm has shifted toward direct methods that avoid estimating the individual densities.

Early direct methods recast the problem as probabilistic classification. By training a binary classifier (e.g., logistic regression) to discriminate between samples from the two distributions, the density ratio can be recovered from the resulting class-posterior probabilities \citep{Qin1998InferencesCS, Bickel2009DiscriminativeLU}. In parallel, moment matching techniques have emerged as a powerful alternative, with Kernel Mean Matching (KMM) \citep{Gretton2008CovariateSK} being the most prominent example. KMM circumvents explicit density modeling by directly computing the density ratio that align the kernel mean embeddings of the two distributions in a Reproducing Kernel Hilbert Space (RKHS). The elegance of this principle has motivated a wealth of recent methodological and theoretical advances in density ratio estimation \citep{Gizewski2022RegularizationUD, Gizewski2026ImpactSK, Myleiko2025RecoveringRD, Liu2026UnboundedDR}.

A more general and theoretically richer framework for direct estimation is based on divergence minimization and ratio fitting. Methods in this family optimize the ratio by minimizing a chosen statistical divergence. Prominent examples include the Kullback-Leibler Importance Estimation Procedure (KLIEP) \citep{Sugiyama2008DirectIE}, which minimizes the KL divergence, and Least-Squares Importance Fitting (LSIF) \citep{Kanamori2009LeastAD}, which minimizes the squared loss.

More recently, the literature addresses high-dimensional and highly discrepant scenarios through structural and architectural innovations. This progress includes telescoping and path-based estimators that bridge large density chasms \citep{Rhodes2020TelescopingDE, Choi2022DensityRE}, as well as methods that explicitly handle unbounded density ratios \citep{Feng2024DeepNQ, Xu2025EstimatingUD, Zheng2026ErrorAD, Liu2026UnboundedDR}.

Following this line of research, we exploit the identity \( \UXY \, \etaA = \uT \) derived in \eqref{eq: key relation 1}, where \( \uT \) is the kernel mean embedding of the test input distribution. Estimating the density ratio \( \etaA \) therefore reduces to solving a regularized operator equation, which we solve using a spectral algorithm. This construction can be viewed as a structured variant of kernel mean matching adapted to the target shift setting, and it naturally unifies several common regularization schemes.

\subsection{Discrete and Continuous Target Shift}
Under target shift, the conditional distribution of the input given the output remains invariant across domains, while the marginal distribution of the output changes. As a consequence, the density ratio depends only on the output variable, substantially simplifying the structure of the distribution shift. This setting has motivated a considerable body of work, particularly in classification and, more recently, in regression.

Most existing studies on target shift have focused on the \emph{discrete} case, i.e., classification problems in which the output variable takes values in a finite set. In this setting, the importance weights reduce to a finite-dimensional vector indexed by the class labels, making both estimation and theoretical analysis considerably more tractable. Early work by \citet{Saerens2002AdjustingOC} introduced an Expectation-Maximization (EM) algorithm to estimate the target label distribution, assuming a predictor with well-calibrated class probabilities. \citet{Lipton2018DetectingCL} relaxed this calibration requirement by proposing Black Box Shift Estimation (BBSE), which solves a linear system based on the confusion matrix of an arbitrary predictor, and established consistency along with finite-sample error bounds. Building on this framework, \citet{Azizzadenesheli2019RegularizedLD} added regularization to improve stability with limited data and derived generalization bound for importance-weighted classifiers under label shift without labeled test data. From an empirical perspective, \citet{Alexandari2020MaximumLB}  showed that combining modern calibration techniques (e.g., bias-corrected temperature scaling) with the EM algorithm often yields better performance. Subsequently, \citet{Garg2020UnifiedVL}  unified these approaches, revealing that confusion matrix estimators implicitly perform a coarse calibration, and proved that maximum likelihood estimation is consistent under the same invertibility conditions as BBSE when a calibrated predictor is used. Taken together, these developments provide a mature theoretical and algorithmic framework for discrete target shift, where the finite-dimensional structure of the label space greatly facilitates both estimator construction and statistical analysis.

In the \emph{continuous} setting, corresponding to regression problems, the importance weight becomes an unknown function that need to be estimated from data, and its characterization leads to an infinite-dimensional integral equation that is inherently ill-posed \citep{Kress2014LinearIE}. Consequently, small perturbations in the data may induce large deviations in an unregularized solution, making regularization essential. Despite the practical relevance of regression under target shift, the continuous setting has received comparatively limited attention. The existing literature is largely algorithmic. \citet{Zhang2013DomainAU} extended kernel mean matching to target shift by matching the kernel mean embedding of the test input distribution with that of a reweighted training distribution, leading to a quadratic program that estimates the density ratio at the observed training outputs. \citet{Nguyen2016ContinuousTS} instead proposed estimating the density ratio function directly by minimizing an \(\calL^{2}\)-based discrepancy between the corresponding distributions, resulting in a constrained optimization problem. More recently, \citet{Kim2024ReTaSANF} developed a nonparametric regularized approach that formulates density ratio estimation as a Tikhonov-regularized integral equation, with the unknown distributions approximated by kernel density estimators, and established consistency of the resulting estimator.

These pioneering studies demonstrate the practical feasibility of adaptation under continuous target shift. Nevertheless, a substantial theoretical gap remains: existing methods do not provide explicit sample-size-dependent convergence rates for the estimated density ratio, let alone establish their optimality. For example, although \citet{Kim2024ReTaSANF} proved consistency and derived convergence bounds depending on the bandwidth and regularization parameters, these parameters were not optimized with respect to the sample size, and no rate-optimality result was established. In parallel, \citet{Gogolashvili2026ImportanceWC} recently derived minimax-optimal rates for importance-weighted kernel ridge regression under target shift, but assumed that the true density ratio is known and thus did not address the density ratio estimation problem. Consequently, a rigorous finite-sample theory for continuous density ratio estimation, providing explicit convergence rates together with an understanding of their optimality, has remained lacking.

\subsection{Discussions on Our Results}
Our results address the theoretical gap identified above by providing  explicit finite-sample convergence guarantees for continuous density ratio estimation under target shift.
In particular, \thmref{thm: density ratio} shows that the proposed spectral estimator achieves the capacity-independent minimax-optimal rate \( O(\neta^{-\iota / (2 \iota + 2)}) \) in the RKHS-norm for estimating the unknown density ratio. Here \( \iota \in (0, \tau] \) is the regularity parameter (\aspref{asp: source etaA}) and \( \neta \) denotes the common size of the labeled training sample and the unlabeled test sample used for ratio estimation. To the best of our knowledge, this provides the first explicit rate-optimal finite-sample guarantee for density ratio estimation under continuous target shift, thereby strengthening the existing consistency results.

\thmref{thm: regression function} further quantifies how the estimated ratio affects the downstream importance-weighted regression estimator. When the sample size  used for density ratio estimation is sufficiently large relative to that used for regression, namely, \( \neta^{2 \iota / (2 \iota + 2)} \ge \nf \), the resulting regression estimator achieves the capacity-independent minimax-optimal rates for standard kernel regression: \( O(\nf^{-(r + 1 / 2) / (2 r + 2)}) \) in the \( \calL^{2} \)-norm and \( O(\nf^{-r / (2 r + 2)}) \) in the RKHS-norm, where \( r \in (0, \tau - 1/2] \) is the regularity parameter of the regression function specified in (\aspref{asp: source fTa}). In the opposite regime, the overall convergence rate is limited by the accuracy of the estimated density ratio and is therefore slower than the minimax-optimal regression rate. These results provide a finite-sample theoretical foundation for importance-weighted regression under continuous target shift and clarify how the statistical accuracy of density ratio estimation affects downstream prediction.

In addition, our analysis also has several limitations that point to directions for future research.
\begin{itemize}
    \item \aspref{asp: source etaA} requires the density ratio to lie in the range of a fractional power of \( T = \UYX \, \UXY \), which excludes components of \( \HY \) that are only weakly coupled to \( \HX \) through the cross-covariance structure. This condition, while essential for our analysis, is more restrictive than typical regularity assumptions in standard kernel regression. It therefore remains an interesting question whether alternative estimation principles can achieve comparable convergence guarantees under weaker regularity conditions.

\item As discussed in \secref{sec: result}, our convergence guarantees for the density ratio are established in the RKHS norm, whereas corresponding rates in the \( \calL^{2} \)-norm are not derived. Establishing \( \calL^{2} \)-convergence rates, and determining whether they are minimax optimal, is an important direction for future work.

    \item \thmref{thm: regression function} shows that the regression estimator achieves minimax-optimal rates only when \( \neta^{2 \iota / (2 \iota + 2)} \ge \nf \), i.e., the sample size used for density ratio estimation need to scale polynomially with the regression sample size. In particular, increasing \( \neta \) requires additional labeled training samples, which may be costly to obtain. Developing more sample-efficient density ratio estimation and importance-weighting procedures that relax this sample-size requirement is therefore an important direction for future research.

\end{itemize}

\section{Proofs of Main Results}
\label{sec: proof}

This section contains the proofs of \thmref{thm: density ratio} and \thmref{thm: regression function}. The argument is based on an error decomposition: the target quantity is split into several components, each component is bounded separately by an auxiliary proposition, and the resulting bounds are combined to obtain the final estimate.

\subsection{Proof of \thmref{thm: density ratio}}
In this subsection, we prove our first main result, Theorem \ref{thm: density ratio}. We begin by deriving an error decomposition for   \( \norm{ \etaA - \hetamu }_{\HY} \). To this end,  we introduce the auxiliary function
\[
    \etamu = \gmu(T) \, T \, \etaA
    = \gmu(T) \, \UYX \, \UXY \, \etaA
    = \gmu(T) \, \UYX \, \uT,
\]
where the equality \( \UXY \, \etaA = \uT \) follows from \eqref{eq: key relation 1}. Then we can write
\[
    \etaA - \hetamu = (\etaA - \etamu) + (\etamu - \hetamu).
\]
Recalling that \( \hetamu = \gmu (\hT) \, \hUYX \, \huT \), we expand \( \etamu - \hetamu \) as
\begin{align*}
    \etamu - \hetamu
     & = \etamu - \gmu (\hT) \, \hUYX \, \huT                                                \\
     & = (\IY - \gmu(\hT) \, \hT + \gmu(\hT) \, \hT) \, \etamu - \gmu (\hT) \, \hUYX \, \huT \\
     & = \gmu(\hT) \, (\hT \, \etamu - \hUYX \, \huT) + (\IY - \gmu(\hT) \, \hT) \, \etamu,
\end{align*}
where \( \IY \) denotes the identity map on \( \HY \). The term \( \hT \, \etamu - \hUYX \, \huT \) can be further decomposed as
\[
    \hT \, \etamu - \hUYX \, \huT
    = \LE( (\hT \, \etamu - \hUYX \, \huT) - (T \, \etamu - \UYX \, \uT) \RI) + (T \, \etamu - \UYX \, \uT).
\]

Summarizing, we obtain
\[
    \norm{ \etaA - \hetamu }_{\HY} \le A_{1} + A_{2} + A_{3} + A_{4},
\]
where
\begin{equation}
    \label{eq: A1 A2 A3 A4}
    \begin{aligned}
         & A_{1} = \norm{ \etaA - \etamu }_{\HY},
        \quad A_{2} = \norm{ \gmu(\hT) \LE( (\hT \, \etamu - \hUYX \, \huT) - (T \, \etamu - \UYX \, \uT) \RI) }_{\HY}, \\
         & A_{3} = \norm{ \gmu(\hT) \, (T \, \etamu - \UYX \, \uT) }_{\HY},
        \quad A_{4} = \norm{ (\IY - \gmu(\hT) \, \hT) \, \etamu }_{\HY}.
    \end{aligned}
\end{equation}
We now bound \( A_{1} \), \( A_{2} \), \( A_{3} \), and \( A_{4} \) individually via separate propositions; the bounds will later be combined. As a preliminary tool, we first establish a proposition controlling empirical averages, which will be used repeatedly throughout the proof.
\begin{proposition}
    \label{pro: A0}
    For any \( \delta \in (0, 1) \), the following bounds hold simultaneously with probability at least \( 1 - \delta \):
    \[
        \norm{ \UYX - \hUYX }_{\HX \to \HY} \le 10 \kapX \kapY \cdot \neta^{-1 / 2} \log \frac{4}{\delta},
        \quad \norm{ \uT - \huT }_{\HX} \le 10 \kapX \cdot \neta^{-1 / 2} \log \frac{4}{\delta}.
    \]
\end{proposition}
\begin{proof}
    Define the random variables
    \[
        \xi_{1}(\bsx, y) = \KY(\cdot, y) \otimes \KX(\cdot, \bsx),
        \quad (\bsx, y) \sim \pSXtY,
    \]
    and
    \[
        \xi_{2}(\bsx') =\KX(\cdot, \bsx'),
        \quad \bsx' \sim \pTX,
    \]
    and set \( \xi_{1, i} = \xi_{1}(\bsx_{i}, y_{i}) \), \( \xi_{2, i} = \xi_{2}(\bsx'_{i}) \) for \( 1 \le i \le \neta \). Then
    \begin{align*}
        \UYX - \hUYX & = \EE[(\bsx, y) \sim \pSXtY]{ \xi_{1}(\bsx, y) } - \frac{1}{\neta} \sum_{i=1}^{\neta} \xi_{1, i}, \\
        \uT - \huT   & = \EE[\bsx' \sim \pTX]{ \xi_{2}(\bsx') } - \frac{1}{\neta} \sum_{i=1}^{\neta} \xi_{2, i}.
    \end{align*}

    By construction, the random variables satisfy the following uniform bounds:
    \[
        \sup_{\bsx \in \HX, y \in \HY} \norm{ \xi_{1}(\bsx, y) }_{\mathrm{HS}(\HX, \HY)} \le \kapX \kapY,
        \quad \sup_{\bsx' \in \calX} \norm{ \xi_{2}(\bsx') }_{\HX} \le \kapX,
    \]
    where \( \mathrm{HS}(\HX, \HY) \) denotes the space of Hilbert–Schmidt operators from \( \HX \) to \( \HY \). Applying \lemref{lem: concentration vector}, each of the two inequalities
    \begin{align*}
         & \norm{ \UYX - \hUYX }_{\HX \to \HY} \le \norm{ \UYX - \hUYX }_{\mathrm{HS}(\HX, \HY)} \le 10 \kapX \kapY \cdot \neta^{-1 / 2} \log \frac{4}{\delta}, \\
         & \norm{ \uT - \huT }_{\HX} \le 10 \kapX \cdot \neta^{-1 / 2} \log \frac{4}{\delta},
    \end{align*}
    holds with probability at least \( 1 - \delta/2 \). By the union bound, both bounds hold simultaneously with probability at least \( 1 - \delta \). This completes the proof.
\end{proof}

Now we present \proref{pro: A1}, which provides an upper bound for \( A_{1} \) in \eqref{eq: A1 A2 A3 A4}.
\begin{proposition}
    \label{pro: A1}
    Suppose that \aspref{asp: source etaA} holds with \( \iota \in (0, \tau] \). Then the following bound holds:
    \[
        A_{1} = \norm{ \etaA - \etamu }_{\HY} \le F \norm{ \piA }_{\HY} \cdot \mu^{\iota}.
    \]
\end{proposition}
\begin{proof}
    Recall that \( \etamu = \gmu(T) \, T \, \etaA \), and by \aspref{asp: source etaA} we have \( \etaA = T^{\iota} \, \piA \). It follows that
    \begin{align*}
        A_{1}
         & = \norm{ \etaA - \etamu }_{\HY}
        = \norm{ (\IY - \gmu(T) \, T) \, \etaA }_{\HY}
        = \norm{ (\IY - \gmu(T) \, T) \, T^{\iota} \, \piA }_{\HY}                                 \\
         & \le \norm{ (\IY - \gmu(T) \, T) \, T^{\iota} }_{\HY \to \HY} \cdot \norm{ \piA }_{\HY}.
    \end{align*}

    Using the filter function property in \eqref{eq: filter F}, we obtain the following bound on the operator norm:
    \[
        \norm{ (\IY - \gmu(T) \, T) \, T^{\iota} }_{\HY \to \HY}
        \le \sup_{0 \le t \le (\kapX \kapY)^{2}} |1 - \gmu(t) t| t^{\iota} \le F \cdot \mu^{\iota}.
    \]
    Consequently,
    \[
        A_{1} \le \norm{ (\IY - \gmu(T) \, T) \, T^{\iota} }_{\HY \to \HY} \cdot \norm{ \piA }_{\HY} \le F \norm{ \piA }_{\HY} \cdot \mu^{\iota},
    \]
    which completes the proof.
\end{proof}

The bound for \( A_{2} \) in \eqref{eq: A1 A2 A3 A4} is given in \proref{pro: A2}.
\begin{proposition}
    \label{pro: A2}
     Suppose that \aspref{asp: source etaA} holds with \( \iota \in (0, \tau] \), then for any \( \delta \in (0, 1) \), the following bounds hold with probability at least \( 1 - \delta \):
    \begin{align*}
        A_{2}
         & = \norm{ \gmu(\hT) \LE( (\hT \, \etamu - \hUYX \, \huT) - (T \, \etamu - \UYX \, \uT) \RI) }_{\HY}                \\
         & \le 40 E^{2} (\kapX \kapY)^{2 \iota + 2} \norm{ \piA }_{\HY} \cdot \mu^{-1} \neta^{-1 / 2} \log \frac{4}{\delta}.
    \end{align*}
\end{proposition}
\begin{proof}
    Observing that
    \begin{align*}
         & \eqspace (\hT \, \etamu - \hUYX \, \huT) - (T \, \etamu - \UYX \, \uT) \\
         & = (\hT - T) \, \etamu + (\UYX - \hUYX) \, \uT + \hUYX \, (\uT - \huT),
    \end{align*}
    we can bound \( A_{2} \) as
    \[
        A_{2} \le \norm{ \gmu(\hT) }_{\HY \to \HY} \cdot (A_{2, 1} + A_{2, 2} + A_{2, 3}),
    \]
    where
    \[
        A_{2, 1} = \norm{ (\hT - T) \, \etamu }_{\HY},
        \quad A_{2, 2} = \norm{ (\UYX - \hUYX) \, \uT }_{\HY},
        \quad A_{2, 3} = \norm{ \hUYX \, (\uT - \huT) }_{\HY}.
    \]

    By the filter function property in \eqref{eq: filter E}, we have
    \begin{equation}
        \label{eq: A20}
        \norm{ \gmu(\hT) }_{\HY \to \HY} \le \sup_{0 \le t \le (\kapX \kapY)^{2}} \gmu(t) \le E \cdot \mu^{-1}.
    \end{equation}
    We now bound \( A_{2, 1} \), \( A_{2, 2} \), and \( A_{2, 3} \) as follows:
    \begin{enumerate}
        \item For \( A_{2, 1} \), note that \( T - \hT \) can be expressed as
              \[
                  T - \hT = \UYX \, \UXY - \hUYX \, \hUXY = \UYX \, (\UXY - \hUXY) + (\UYX - \hUYX) \, \hUXY.
              \]
              Consequently,
              \begin{align*}
                  A_{2, 1}
                   & = \norm{ (\hT - T) \, \etamu }_{\HY}
                  \le \norm{ \hT - T }_{\HY \to \HY} \cdot \norm{ \etamu }_{\HY}                                                                        \\
                   & \le \LE( \norm{ \UYX \, (\UXY - \hUXY) }_{\HY \to \HY} + \norm{ (\UYX - \hUYX) \, \hUXY }_{\HY \to \HY} \RI) \norm{ \etamu }_{\HY} \\
                   & \le 2 \kapX \kapY \cdot \norm{ \UYX - \hUYX }_{\HX \to \HY} \cdot \norm{ \etamu }_{\HY}.
              \end{align*}
              In the last inequality, we used the bounds \( \norm{ \UYX }_{\HX \to \HY} \le \kapX \kapY \) and \( \norm{ \hUXY }_{\HY \to \HX} \le \kapX \kapY \), together with the identity
              \[
                  \norm{ \UXY - \hUXY }_{\HY \to \HX} = \norm{ \UYX - \hUYX }_{\HX \to \HY}.
              \]

              By \proref{pro: A0}, we have
              \[
                  \norm{ \UYX - \hUYX }_{\HX \to \HY} \le 10 \kapX \kapY \cdot \neta^{-1 / 2} \log \frac{4}{\delta}
              \]
              holds with probability at least $1-\delta.$
              Furthermore, under Assumption \ref{asp: source etaA} with $0<\iota\le \tau $, we have 
              \begin{align*}
                  \norm{ \etamu }_{\HY}
                   & = \norm{ \gmu(T) \, T^{\iota + 1} \, \piA }_{\HY}
                  \le \norm{ \gmu(T) \, T }_{\HY \to \HY} \cdot \norm{ T^{\iota} }_{\HY \to \HY} \cdot \norm{ \piA }_{\HY} \\
                   & \le E (\kapX \kapY)^{2 \iota} \norm{ \piA }_{\HY}.
              \end{align*}
              The last inequality follows from \( \norm{ T }_{\HY \to \HY} \le (\kapX \kapY)^{2} \) and, by the filter function property in \eqref{eq: filter E},
              \[
                  \norm{ \gmu(T) \, T }_{\HY \to \HY} \le \sup_{0 \le t \le (\kapX \kapY)^{2}} \gmu(t) t \le E.
              \]

              Combining these estimates yields
              \begin{equation}
                  \label{eq: A21}
                  \begin{aligned}
                      A_{2, 1}
                       & \le 2 \kapX \kapY \cdot \norm{ \UYX - \hUYX }_{\HX \to \HY} \cdot \norm{ \etamu }_{\HY}              \\
                       & \le 20 E (\kapX \kapY)^{2 \iota + 2} \norm{ \piA }_{\HY} \cdot \neta^{-1 / 2} \log \frac{4}{\delta}.
                  \end{aligned}
              \end{equation}

        \item Since the bound \( \norm{ \uT }_{\HX} \le \kapX \) holds trivially, applying \proref{pro: A0} gives
              \begin{equation}
                  \label{eq: A22}
                  \begin{aligned}
                      A_{2, 2}
                       & = \norm{ (\UYX - \hUYX) \, \uT }_{\HY}
                      \le \norm{ \UYX - \hUYX }_{\HX \to \HY} \cdot \norm{ \uT }_{\HX}      \\
                       & \le 10 \kapX^{2} \kapY \cdot \neta^{-1 / 2} \log \frac{4}{\delta}.
                  \end{aligned}
              \end{equation}

        \item For \( A_{2, 3} \), using the bound \( \norm{ \hUYX }_{\HX \to \HY} \le \kapX \kapY \) together with \proref{pro: A0} yields
              \begin{equation}
                  \label{eq: A23}
                  \begin{aligned}
                      A_{2, 3}
                       & = \norm{ \hUYX \, (\uT - \huT) }_{\HY}
                      \le \norm{ \hUYX }_{\HX \to \HY} \cdot \norm{ \uT - \huT }_{\HX}      \\
                       & \le 10 \kapX^{2} \kapY \cdot \neta^{-1 / 2} \log \frac{4}{\delta}.
                  \end{aligned}
              \end{equation}
    \end{enumerate}

    Combining the estimates in \eqref{eq: A20}, \eqref{eq: A21}, \eqref{eq: A22}, and \eqref{eq: A23} yields
    \begin{align*}
        A_{2} & \le \norm{ \gmu(\hT) }_{\HY \to \HY} \cdot (A_{2, 1} + A_{2, 2} + A_{2, 3})                                       \\
              & \le 40 E^{2} (\kapX \kapY)^{2 \iota + 2} \norm{ \piA }_{\HY} \cdot \mu^{-1} \neta^{-1 / 2} \log \frac{4}{\delta},
    \end{align*}
    which holds with probability at least $1-\delta.$    This completes the proof.
\end{proof}

Next, we bound \( A_{3} \) in \eqref{eq: A1 A2 A3 A4} using \proref{pro: A3}.
\begin{proposition}
    \label{pro: A3}
     Suppose that \aspref{asp: source etaA} holds with \( \iota \in (0, \tau] \), and that the bounds in \proref{pro: A0} are satisfied. If the sample size \( \neta \) and the regularization parameter \( \mu \) satisfy
    \[
        \mu \neta^{1 / 2} \ge 40 (\kapX \kapY)^{2} \log \frac{4}{\delta},
    \]
    then the following bound holds:
    \[
        A_{3} = \norm{ \gmu(\hT) \, (T \, \etamu - \UYX \, \uT) }_{\HY} \le 4 E F \norm{ \piA }_{\HY} \cdot \mu^{\iota}.
    \]
\end{proposition}
\begin{proof}
    Inserting the identities \( (\hT + \mu \IY) \, (\hT + \mu \IY)^{-1} \) and \( (T + \mu \IY) \, (T + \mu \IY)^{-1} \), we decompose \( A_{3} \) as
    \[
        A_{3} \le A_{3, 1} \cdot A_{3, 2} \cdot A_{3, 3},
    \]
    where
    \begin{align*}
         & A_{3, 1} = \norm{ \gmu(\hT) \, (\hT + \mu \IY) }_{\HY \to \HY},
        \quad A_{3, 2} = \norm{ (\hT + \mu \IY)^{-1} \, (T + \mu \IY) }_{\HY \to \HY},  \\
         & A_{3, 3} = \norm{ (T + \mu \IY)^{-1} \, (T \, \etamu - \UYX \, \uT) }_{\HY}.
    \end{align*}

    We now bound \( A_{3,1} \), \( A_{3,2} \), and \( A_{3,3} \) as follows:
    \begin{enumerate}
        \item By the filter function property in \eqref{eq: filter E}, we obtain
              \begin{equation}
                  \label{eq: A31}
                  \begin{aligned}
                      A_{3, 1}
                       & = \norm{ \gmu(\hT) \, (\hT + \mu \IY) }_{\HY \to \HY}
                      \le \sup_{0 \le t \le (\kapX \kapY)^{2}} \gmu(t) (t + \mu) \\
                       & \le E + E \cdot \mu^{-1} \cdot \mu
                      = 2 E.
                  \end{aligned}
              \end{equation}

        \item Since
              \[
                  \hT + \mu \IY = (T + \mu \IY) \LE( \IY - (T + \mu \IY)^{-1} \, (T - \hT) \RI),
              \]
              we have
              \[
                  A_{3, 2} = \norm{ (\hT + \mu \IY)^{-1} \, (T + \mu \IY) }_{\HY \to \HY}
                  = \norm{ \LE( \IY - (T + \mu \IY)^{-1} \, (T - \hT) \RI)^{-1} }_{\HY \to \HY}.
              \]

              To bound the norm of the inverse, we first estimate \( \norm{ (T + \mu \IY)^{-1} \, (T - \hT) }_{\HY \to \HY} \). As shown in the proof of \proref{pro: A2} (when bounding \( A_{2,1} \)), we have
              \[
                  \norm{ \hT - T }_{\HY \to \HY}
                  \le 2 \kapX \kapY \cdot \norm{ \UYX - \hUYX }_{\HX \to \HY}
                  \le 20 (\kapX \kapY)^{2} \cdot \neta^{-1 / 2} \log \frac{4}{\delta},
              \]
              under the assumption that the bounds in \proref{pro: A0} hold. Thus,
              \begin{align*}
                  \norm{ (T + \mu \IY)^{-1} \, (T - \hT) }_{\HY \to \HY}
                   & \le \norm{ (T + \mu \IY)^{-1} }_{\HY \to \HY} \cdot \norm{ \hT - T }_{\HY \to \HY} \\
                   & \le 20 (\kapX \kapY)^{2} \cdot \mu^{-1} \neta^{-1 / 2} \log \frac{4}{\delta},
              \end{align*}
              where we use \( \norm{ (T + \mu \IY)^{-1} }_{\HY \to \HY} \le \mu^{-1} \) in the last inequality.

              If the sample size \( \neta \) and the regularization parameter \( \mu \) satisfy
              \[
                  \mu \neta^{1 / 2} \ge 40 (\kapX \kapY)^{2} \log \frac{4}{\delta},
              \]
              then
              \[
                  \norm{ (T + \mu \IY)^{-1} \, (T - \hT) }_{\HY \to \HY} \le 1 / 2.
              \]
              By the Neumann series expansion, we obtain
              \begin{equation}
                  \label{eq: A32}
                  A_{3, 2} = \norm{ \LE( \IY - (T + \mu \IY)^{-1} \, (T - \hT) \RI)^{-1} }_{\HY \to \HY}
                  \le \frac{1}{1 - 1 / 2} = 2.
              \end{equation}

        \item For \( A_{3,3} \), we use the relation
              \[
                  T \, \etamu - \UYX \, \uT
                  = T \, \etamu - \UYX \, \UXY \, \etaA
                  = T \, \etamu - T \, \etaA
              \]
              to obtain
              \begin{equation}
                  \label{eq: A33}
                  \begin{aligned}
                      A_{3, 3}
                       & = \norm{ (T + \mu \IY)^{-1} \, (T \, \etamu - \UYX \, \uT) }_{\HY}
                      = \norm{ (T + \mu \IY)^{-1} \, T \, (\etamu - \etaA) }_{\HY}                              \\
                       & \le \norm{ (T + \mu \IY)^{-1} \, T }_{\HY \to \HY} \cdot \norm{ \etamu - \etaA }_{\HY}
                      \le F \norm{ \piA }_{\HY} \cdot \mu^{\iota}.
                  \end{aligned}
              \end{equation}
              In the last inequality, we use \( \norm{ (T + \mu \IY)^{-1} \, T }_{\HY \to \HY} \le 1 \) and the bound for \( \norm{ \etamu - \etaA }_{\HY} \) from \proref{pro: A1}.
    \end{enumerate}

    Combining the estimates in \eqref{eq: A31}, \eqref{eq: A32}, and \eqref{eq: A33}, we conclude that when
    \[
        \mu \neta^{1 / 2} \ge 40 (\kapX \kapY)^{2} \log \frac{4}{\delta},
    \]
    the following inequality holds:
    \[
        A_{3} \le A_{3, 1} \cdot A_{3, 2} \cdot A_{3, 3}
        \le 4 E F \norm{ \piA }_{\HY} \cdot \mu^{\iota}.
    \]
    This completes the proof.
\end{proof}

Finally, \proref{pro: A4} provides a bound for \( A_{4} \) in \eqref{eq: A1 A2 A3 A4}.
\begin{proposition}
    \label{pro: A4}
    Assume that the bounds in \proref{pro: A0} hold. Then we have
    \begin{align*}
        A_{4}
         & = \norm{ (\IY - \gmu(\hT) \, \hT) \, \etamu }_{\HY}                                                                                                                 \\
         & \le E F \norm{ \piA }_{\HY} \LE( \mu^{\iota} + 20 (\iota + 1) (\kapX \kapY)^{2 \iota} \cdot \neta^{-\frac{\min \family{ \iota, 1 }}{2}} \RI) \log \frac{4}{\delta}.
    \end{align*}
\end{proposition}
\begin{proof}
    Since \( \etamu = \gmu(T) \, T \, \etaA = \gmu(T) \, T^{\iota + 1} \, \piA \), we have
    \begin{align*}
        A_{4}
         & = \norm{ (\IY - \gmu(\hT) \, \hT) \, \etamu }_{\HY}
        = \norm{ (\IY - \gmu(\hT) \, \hT) \, \gmu(T) \, T^{\iota + 1} \, \piA }_{\HY}                                                           \\
         & \le \norm{ (\IY - \gmu(\hT) \, \hT) \, T^{\iota} }_{\HY \to \HY} \cdot \norm{ \gmu(T) \, T }_{\HY \to \HY} \cdot \norm{ \piA }_{\HY} \\
         & \le \norm{ (\IY - \gmu(\hT) \, \hT) \, T^{\iota} }_{\HY \to \HY} \cdot E \cdot \norm{ \piA }_{\HY},
    \end{align*}
    where the filter function property \eqref{eq: filter E} gives
    \[
        \norm{ \gmu(T) \, T }_{\HY \to \HY}
        \le \sup_{0 \le t \le (\kapX \kapY)^{2}} \gmu(t) t \le E.
    \]
    To bound the remaining operator norm, we add and subtract \( \hT^{\iota} \):
    \begin{align*}
         & \eqspace \norm{ (\IY - \gmu(\hT) \, \hT) \, T^{\iota} }_{\HY \to \HY}                                                                                                     \\
         & \le \norm{ (\IY - \gmu(\hT) \, \hT) \, \hT^{\iota} }_{\HY \to \HY} + \norm{ (\IY - \gmu(\hT) \, \hT) }_{\HY \to \HY} \cdot \norm{ T^{\iota} - \hT^{\iota} }_{\HY \to \HY} \\
         & \le F \LE( \mu^{\iota} + \norm{ T^{\iota} - \hT^{\iota} }_{\HY \to \HY} \RI).
    \end{align*}
    Here we used the filter function property \eqref{eq: filter F} to obtain
    \begin{align*}
         & \norm{ (\IY - \gmu(\hT) \, \hT) \, \hT^{\iota} }_{\HY \to \HY}
        \le \sup_{0 \le t \le (\kapX \kapY)^{2}} |1 - \gmu(t) t| t^{\iota} \le F \cdot \mu^{\iota},                        \\
         & \norm{ (\IY - \gmu(\hT) \, \hT) }_{\HY \to \HY} \le \sup_{0 \le t \le (\kapX \kapY)^{2}} |1 - \gmu(t) t| \le F.
    \end{align*}

    For the term \( \norm{ T^{\iota} - \hT^{\iota} }_{\HY \to \HY} \), we invoke \lemref{lem: A^c - B^c}, which yields
    \[
        \norm{ T^{\iota} - \hT^{\iota} }_{\HY \to \HY} \le
        \begin{cases}
            \norm{ T - \hT }_{\HY \to \HY}^{\iota},                           & \iota \in (0, 1]; \\
            \iota (\kapX \kapY)^{2 \iota - 2} \norm{ T - \hT }_{\HY \to \HY}, & \iota > 1.
        \end{cases}
    \]
    As established in the proof of \proref{pro: A2} (when bounding \( A_{2,1} \)), under the assumptions of \proref{pro: A0} we have
    \[
        \norm{ \hT - T }_{\HY \to \HY}
        \le 20 (\kapX \kapY)^{2} \cdot \neta^{-1 / 2} \log \frac{4}{\delta}.
    \]
    Combining this with the above case distinction, gives
    \[
        \norm{ T^{\iota} - \hT^{\iota} }_{\HY \to \HY}
        \le 20 (\iota + 1) (\kapX \kapY)^{2 \iota} \cdot \neta^{-\frac{\min \family{ \iota, 1 }}{2}} \log \frac{4}{\delta}.
    \]

    Putting everything together, we obtain
    \begin{align*}
        A_{4}
         & \le \norm{ (\IY - \gmu(\hT) \, \hT) \, T^{\iota} }_{\HY \to \HY} \cdot E \cdot \norm{ \piA }_{\HY}                                                                  \\
         & \le E F \norm{ \piA }_{\HY} \LE( \mu^{\iota} + \norm{ T^{\iota} - \hT^{\iota} }_{\HY \to \HY} \RI)                                                                  \\
         & \le E F \norm{ \piA }_{\HY} \LE( \mu^{\iota} + 20 (\iota + 1) (\kapX \kapY)^{2 \iota} \cdot \neta^{-\frac{\min \family{ \iota, 1 }}{2}} \RI) \log \frac{4}{\delta},
    \end{align*}
    which completes the proof.
\end{proof}

With the bounds for \( A_{1} \), \( A_{2} \), \( A_{3} \), and \( A_{4} \) established, we are now ready to prove \thmref{thm: density ratio}.
\begin{proofof}{\thmref{thm: density ratio}}
    We set the regularization parameter as \( \mu = \neta^{-1 / (2 \iota + 2)} \). Consequently,
    \[
        \max \family{ \mu^{\iota}, \mu^{-1} \neta^{-1 / 2}, \neta^{-\frac{\min \family{ \iota, 1 }}{2}} } \le \neta^{-\frac{\iota}{2 \iota + 2}}.
    \]
    Furthermore, for any \( \delta \in (0, 1) \), if the sample size \( \neta \) is sufficiently large such that
    \[
        \neta \ge \LE( 40 (\kapX \kapY)^{2} \log \frac{4}{\delta} \RI)^{\frac{2 \iota + 2}{\iota}},
    \]
    then, on the event described in \proref{pro: A0} (which holds with probability at least \( 1 - \delta \)), the bounds in Propositions~\ref{pro: A1}--\ref{pro: A4} hold simultaneously. This yields the following inequalities:
    \begin{align*}
         & A_{1} \le F \norm{ \piA }_{\HY} \cdot \neta^{-\frac{\iota}{2 \iota + 2}},
        \quad A_{2} \le 40 E^{2} (\kapX \kapY)^{2 \iota + 2} \norm{ \piA }_{\HY} \cdot \neta^{-\frac{\iota}{2 \iota + 2}} \log \frac{4}{\delta}, \\
         & A_{3} \le 4 E F \norm{ \piA }_{\HY} \cdot \neta^{-\frac{\iota}{2 \iota + 2}},
        \quad A_{4} \le E F \norm{ \piA }_{\HY} \LE( 1 + 20 (\iota + 1) (\kapX \kapY)^{2 \iota} \RI) \neta^{-\frac{\iota}{2 \iota + 2}} \log \frac{4}{\delta}.
    \end{align*}
    Therefore, the error \( \norm{ \etaA - \hetamu }_{\HY} \) can be bounded as
    \[
        \norm{ \etaA - \hetamu }_{\HY} \le A_{1} + A_{2} + A_{3} + A_{4} \le \varDelta_{\eta} \cdot \neta^{-\frac{\iota}{2 \iota + 2}} \log \frac{4}{\delta},
    \]
    where
    \begin{equation}
        \label{eq: Delta_eta}
        \begin{aligned}
            \varDelta_{\eta}
             & = F \norm{ \piA }_{\HY} + 40 E^{2} (\kapX \kapY)^{2 \iota + 2} \norm{ \piA }_{\HY}                                  \\
             & \eqspace + 4 E F \norm{ \piA }_{\HY} + E F \norm{ \piA }_{\HY} \LE( 1 + 20 (\iota + 1) (\kapX \kapY)^{2 \iota} \RI)
        \end{aligned}
    \end{equation}
    is a constant independent of \( \neta \) or \( \delta \). Hence, the theorem holds.
\end{proofof}

\subsection{Proof of \thmref{thm: regression function}}

Having proved \thmref{thm: density ratio}, we know that for any \( \iota \in (0, \tau - 1 / 2] \), \( \hetamu \) converges to \( \etaA \) with high probability. Therefore, our density ratio estimator \( \tetamu = \max\{ \hetamu, 0 \} \) satisfies the uniform bound:
\begin{equation}
    \label{eq: eta bound}
                  \begin{aligned}
                  \sup_{y \in \calY} |\etaA(y) - \tetamu(y)|
                   & \le \sup_{y \in \calY} |\etaA(y) - \hetamu(y)|
                  \le \kapY \cdot \norm{ \etaA - \hetamu }_{\HY}                                                 \\
                   & \le \kapY \varDelta_{\eta} \cdot \neta^{-\frac{\iota}{2 \iota + 2}} \log \frac{4}{\delta'},
              \end{aligned}
\end{equation}
with probability at least \( 1 - \delta' \). The first inequality holds beacuse \( \etaA \) is nonnegative.

Now, we use \( \tetamu \) as a weighting function to construct an estimator \( \hflam \) of the regression function \( \fTA \) via \eqref{eq: hflam}, and analyze its convergence behavior. In this subsection, we establish bounds for \( \norm{ \fTA - \hflam }_{\calL^{2}(\calX, \pTX)} \) and \( \norm{ \fTA - \hflam }_{\HX} \). To unify the analysis for both norms, recall that for any \( f \in \HX \),
\[
    \norm{ f }_{\calL^{2}(\calX, \pTX)} = \norm{ \VXX^{1 / 2} \, f }_{\HX},
\]
where \( \VXX = \EE[\bsx \sim \pTX]{ \KX(\cdot, \bsx) \otimes \KX(\cdot, \bsx) } \) is the covariance operator on the test domain. Consequently, it suffices to bound
\[
    \norm{ \VXX^{\frac{1 - \gamma}{2}} \, (\fTA - \hflam) }_{\HX}
\]
for \( \gamma \in [0, 1] \), and then specialize to \( \gamma = 0 \) or \( 1 \) to obtain the respective norms. In addition, we assume without loss of generality that
\[
    \sup_{y \in \calY} |\etaA(y)| \le M_{\eta},
    \quad \sup_{y \in \calY} |\tetamu(y)| \le M_{\eta},
    \quad \sup_{y \in \calY} |y| \le M_{y}.
\]

We introduce the auxiliary function
\[
    \flam = \glam(\VXX) \, \VXX \, \fTA = \glam(\VXX) \, \vTy,
\]
where the relation \( \VXX \, \fTA = \vTy \) is established in \eqref{eq: key relation 2}. By the triangle inequality,
\[
    \norm{ \VXX^{\frac{1 - \gamma}{2}} \, (\fTA - \hflam) }_{\HX} 
    \le \norm{ \VXX^{\frac{1 - \gamma}{2}} \, (\fTA - \flam) }_{\HX} + \norm{ \VXX^{\frac{1 - \gamma}{2}} \, (\flam - \hflam) }_{\HX}.
\]
For the second term, we have
\begin{align*}
     & \eqspace \norm{ \VXX^{\frac{1 - \gamma}{2}} \, (\flam - \hflam) }_{\HX}                                                                                          \\
     & \le \norm{ \VXX^{\frac{1 - \gamma}{2}} \, (\hWXX + \lambda \IX)^{-1 / 2} }_{\HX \to \HX} \cdot \norm{ (\hWXX + \lambda \IX)^{1 / 2} \, (\flam - \hflam) }_{\HX}.
\end{align*}
Recalling that \( \hflam = \glam(\hWXX) \, \hwy \), we decompose \( \flam - \hflam \) as
\[
    \flam - \hflam = \glam(\hWXX) \, (\hWXX \, \flam - \hwy) + (\IX - \glam(\hWXX) \, \hWXX) \, \flam,
\]
which implies
\begin{align*}
    \norm{ (\hWXX + \lambda \IX)^{1 / 2} \, (\flam - \hflam) }_{\HX}
     & \le \norm{ (\hWXX + \lambda \IX)^{1 / 2} \, \glam(\hWXX) \, (\hWXX \, \flam - \hwy) }_{\HX}        \\
     & \eqspace + \norm{ (\hWXX + \lambda \IX)^{1 / 2} \, (\IX - \glam(\hWXX) \, \hWXX) \, \flam }_{\HX}.
\end{align*}

Combining these results yields the error decomposition
\[
    \norm{ \VXX^{\frac{1 - \gamma}{2}} \, (\fTA - \hflam) }_{\HX}
    \le B_{1} + B_{2} (B_{3} + B_{4}),
\]
where
\begin{equation}
    \label{eq: B1 B2 B3 B4}
    \begin{aligned}
         & B_{1} = \norm{ \VXX^{\frac{1 - \gamma}{2}} \, (\fTA - \flam) }_{\HX},
        \quad B_{2} = \norm{ \VXX^{\frac{1 - \gamma}{2}} \, (\hWXX + \lambda \IX)^{-1 / 2} }_{\HX \to \HX}, \\
         & B_{3} = \norm{ (\hWXX + \lambda \IX)^{1 / 2} \, \glam(\hWXX) \, (\hWXX \, \flam - \hwy) }_{\HX}, \\
         & B_{4} = \norm{ (\hWXX + \lambda \IX)^{1 / 2} \, (\IX - \glam(\hWXX) \, \hWXX) \, \flam }_{\HX}.
    \end{aligned}
\end{equation}
We now bound \( B_{1} \), \( B_{2} \), \( B_{3} \), and \( B_{4} \) individually via separate propositions, and later combine the estimates. As a first step, we give a proposition that controls empirical averages and will be used repeatedly.
\begin{proposition}
    \label{pro: B0}
    Assume that the error bound \eqref{eq: eta bound} for the density ratio estimator \( \tetamu \) holds. Then for any \( \delta \in (0, 1) \), the following bounds hold simultaneously with probability at least \( 1 - \delta \):
    \begin{align*}
         & \norm{ \VXX - \hWXX }_{\HX \to \HX} \le \kapX^{2} \kapY \varDelta_{\eta} \cdot \neta^{-\frac{\iota}{2 \iota + 2}} \log \frac{4}{\delta'} + 10 M_{\eta} \kapX^{2} \cdot \nf^{-1 / 2} \log \frac{4}{\delta}, \\
         & \norm{ (\WXX \, \flam - \wy) - (\hWXX \, \flam - \hwy) }_{\HX} \le 10 M_{\eta} (F \kapX \norm{ \psiA }_{\HX} + 2 M_{y}) \kapX \cdot \nf^{-1 / 2} \log \frac{4}{\delta}.
    \end{align*}
\end{proposition}
\begin{proof}
    We bound these two quantities separately.
    \begin{enumerate}
        \item To bound \( \norm{ \VXX - \hWXX }_{\HX \to \HX} \), we decompose it as
              \[
                  \norm{ \VXX - \hWXX }_{\HX \to \HX}
                  \le \norm{ \VXX - \WXX }_{\HX \to \HX} + \norm{ \WXX - \hWXX }_{\HX \to \HX}.
              \]
              Since
              \begin{align*}
                  \VXX & = \EE[\bsx \sim \pTX]{ \KX(\cdot, \bsx) \otimes \KX(\cdot, \bsx) }
                  = \EE[(\bsx, y) \sim \pTXtY]{ \KX(\cdot, \bsx) \otimes \KX(\cdot, \bsx) }                     \\
                       & = \EE[(\bsx, y) \sim \pSXtY]{ \etaA(y) \, \KX(\cdot, \bsx) \otimes \KX(\cdot, \bsx) },
              \end{align*}
             then the first term \( \norm{ \VXX - \WXX }_{\HX \to \HX} \) can be expressed as
              \begin{align*}
                  \norm{ \VXX - \WXX }_{\HX \to \HX}
                   & = \norm{ \EE[(\bsx, y) \sim \pSXtY]{ (\etaA(y) - \tetamu(y)) \, \KX(\cdot, \bsx) \otimes \KX(\cdot, \bsx) } }_{\HX \to \HX}                         \\
                   & \le \sup_{y \in \calY} |\etaA(y) - \tetamu(y)| \cdot \EE[(\bsx, y) \sim \pSXtY]{ \norm{ \KX(\cdot, \bsx) \otimes \KX(\cdot, \bsx) }_{\HX \to \HX} } \\
                   & \le \kapX^{2} \kapY \varDelta_{\eta} \cdot \neta^{-\frac{\iota}{2 \iota + 2}} \log \frac{4}{\delta'}.
              \end{align*}
              In the last inequality, we use \eqref{eq: eta bound} and the fact that \( \sup_{\bsx \in \calX} \norm{ \KX(\cdot, \bsx) \otimes \KX(\cdot, \bsx) }_{\HX \to \HX} \le \kapX^{2} \). For the second term, \( \norm{ \WXX - \hWXX }_{\HX \to \HX} \), define the random variable
              \[
                  \zeta_{1}(\bsx, y) = \tetamu(y) \, \KX(\cdot, \bsx) \otimes \KX(\cdot, \bsx),
                  \quad (\bsx, y) \sim \pSXtY,
              \]
              and let \( \zeta_{1, j} = \zeta_{1}(\bsx_{j}, y_{j}) \) for \( 1 \le j \le \nf \). Then
              \[
                  \WXX - \hWXX = \EE[(\bsx, y) \sim \pSXtY]{ \zeta_{1}(\bsx, y) } - \frac{1}{\nf} \sum_{j=1}^{\nf} \zeta_{1, j}.
              \]
              Since we assume \( \sup_{y \in \calY} |\tetamu(y)| \le M_{\eta} \), the Hilbert-Schmidt norm of \( \zeta_{1}(\bsx, y) \) is uniformly bounded by \( M_{\eta} \kapX^{2} \). Applying \lemref{lem: concentration vector}, we obtain, with probability at least \( 1 - \delta / 2 \),
              \begin{align*}
                  \norm{ \WXX - \hWXX }_{\HX \to \HX}
                   & \le \norm{ \EE[(\bsx, y) \sim \pSXtY]{ \zeta_{1}(\bsx, y)  } - \frac{1}{\nf} \sum_{j=1}^{\nf} \zeta_{1, j} }_{\mathrm{HS}(\HX, \HX)} \\
                   & \le 10 M_{\eta} \kapX^{2} \cdot \nf^{-1 / 2} \log \frac{4}{\delta}.
              \end{align*}
              Combining these two bounds yields
              \begin{equation}
                  \label{eq: B01}
                  \begin{aligned}
                      \norm{ \VXX - \hWXX }_{\HX \to \HX}
                       & \le \norm{ \VXX - \WXX }_{\HX \to \HX} + \norm{ \WXX - \hWXX }_{\HX \to \HX}                                                                                           \\
                       & \le \kapX^{2} \kapY \varDelta_{\eta} \cdot \neta^{-\frac{\iota}{2 \iota + 2}} \log \frac{4}{\delta'} + 10 M_{\eta} \kapX^{2} \cdot \nf^{-1 / 2} \log \frac{4}{\delta}.
                  \end{aligned}
              \end{equation}

        \item Next, we bound \( \norm{ (\WXX \, \flam - \wy) - (\hWXX \, \flam - \hwy) }_{\HX} \). Define
              \[
                  \zeta_{2}(\bsx, y) = \tetamu(y) (\flam(\bsx) - y) \, \KX(\cdot, \bsx),
                  \quad (\bsx, y) \sim \pSXtY
              \]
              and set \( \zeta_{2, j} = \zeta_{2}(\bsx_{j}, y_{j}) \) for \( 1 \le j \le \nf \). Then
              \[
                  (\WXX \, \flam - \wy) - (\hWXX \, \flam - \hwy) = \EE[(\bsx, y) \sim \pSXtY]{ \zeta_{2}(\bsx, y)  } - \frac{1}{\nf} \sum_{j=1}^{\nf} \zeta_{2, j}.
              \]
              To obtain a uniform bound for \( \norm{ \zeta_{2}(\bsx, y) }_{\HX} \), note that \( \sup_{y \in \calY} |y| \le M_{y} \) implies \( \sup_{x \in \calX} |\fTA(\bsx)| \le M_{y} \). Thus for all \( (\bsx, y) \in \calX \times \calY \),
              \[
                  |\flam(\bsx) - y|
                  \le |\flam(\bsx) - \fTA(\bsx)| + |\fTA(\bsx) - y| \le |\flam(\bsx) - \fTA(\bsx)| + 2 M_{y},
              \]
              and we further bound \( |\flam(\bsx) - \fTA(\bsx)| \) using the reproducing property:
              \begin{align*}
                  |\flam(\bsx) - \fTA(\bsx)| = \LE| \inner{ \flam - \fTA, \KX(\cdot, \bsx) }_{\HX} \RI|
                  \le \kapX \cdot \norm{ \flam - \fTA }_{\HX}
                  \le F \kapX \norm{ \psiA }_{\HX},
              \end{align*}
              where the bound for \( \norm{ \flam - \fTA }_{\HX} \) follows from \proref{pro: B1} with \( \gamma = 1 \) (note that \( \lambda \le 1 \)). Consequently,
              \begin{align*}
                  \sup_{(\bsx, y) \in \calX \times \calY} \norm{ \zeta_{2}(\bsx, y) }_{\HX}
                   & = \sup_{(\bsx, y) \in \calX \times \calY} \norm{ \tetamu(y) (\flam(\bsx) - y) \, \KX(\cdot, \bsx) }_{\HX}              \\
                   & \le \sup_{(\bsx, y) \in \calX \times \calY} |\tetamu(y)| \cdot |\flam(\bsx) - y| \cdot \norm{ \KX(\cdot, \bsx) }_{\HX} \\
                   & \le M_{\eta} (F \kapX \norm{ \psiA }_{\HX} + 2 M_{y}) \kapX.
              \end{align*}
              Applying \lemref{lem: concentration vector} again, we obtain, with probability at least \( 1 - \delta / 2 \),
              \begin{equation}
                  \label{eq: B02}
                  \begin{aligned}
                      \norm{ (\WXX \, \flam - \wy) - (\hWXX \, \flam - \hwy) }_{\HX}
                       & = \norm{ \EE[(\bsx, y) \sim \pSXtY]{ \zeta_{2}(\bsx, y)  } - \frac{1}{\nf} \sum_{j=1}^{\nf} \zeta_{2, j} }_{\HX} \\
                       & \le 10 M_{\eta} (F \kapX \norm{ \psiA }_{\HX} + 2 M_{y}) \kapX \cdot \nf^{-1 / 2} \log \frac{4}{\delta}.
                  \end{aligned}
              \end{equation}
    \end{enumerate}
    The proposition follows from \eqref{eq: B01} and \eqref{eq: B02} via the union bound.
\end{proof}

Next, we present an upper bound for \( B_{1} \) in \eqref{eq: B1 B2 B3 B4} via \proref{pro: B1}.
\begin{proposition}
    \label{pro: B1}
    Suppose that \aspref{asp: source fTa} holds with \( r \in (0, \tau - 1 / 2] \). Then, for any \( \gamma \in [0, 1] \), the following bound holds:
    \[
        B_{1} = \norm{ \VXX^{\frac{1 - \gamma}{2}} \, (\fTA - \flam) }_{\HX} \le F \norm{ \psiA }_{\HX} \cdot \lambda^{r + \frac{1 - \gamma}{2}}.
    \]
\end{proposition}
\begin{proof}
    Recall that \( \flam = \glam(\VXX) \, \VXX \, \fTA \) and \( \fTA = \VXX^{r} \, \psiA \). Using analogous reasoning to the proof of \proref{pro: A1}, we obtain
    \begin{align*}
        B_{1} & = \norm{ \VXX^{\frac{1 - \gamma}{2}} \, (\fTA - \flam) }_{\HX}
        = \norm{ \VXX^{\frac{1 - \gamma}{2}} \, (\IX - \glam(\VXX) \, \VXX) \, \VXX^{r} \, \psiA }_{\HX}                           \\
              & \le \norm{ \VXX^{r + \frac{1 - \gamma}{2}} \, (I - \glam(\VXX) \, \VXX) }_{\HX \to \HX} \cdot \norm{ \psiA }_{\HX}
        \le F \norm{ \psiA }_{\HX} \cdot \lambda^{r + \frac{1 - \gamma}{2}},
    \end{align*}
    where the last inequality follows from the filter function property \eqref{eq: filter F}.
\end{proof}

We bound \( B_{2} \) in \eqref{eq: B1 B2 B3 B4} by \proref{pro: B2}.
\begin{proposition}
    \label{pro: B2}
    Assume that the bounds in \proref{pro: B0} hold. If the sample sizes \( \neta \), \( \nf \) and the regularization parameter \( \lambda \) satisfy
    \[
        \begin{cases}
            \displaystyle \lambda \neta^{\frac{\iota}{2 \iota + 2}} \ge 4 \kapX^{2} \kapY \varDelta_{\eta} \log \frac{4}{\delta'}, \\[1em]
            \displaystyle \lambda \nf^{1 / 2} \ge 40 M_{\eta} \kapX^{2} \log \frac{4}{\delta},
        \end{cases}
    \]
    then for any \( \gamma \in [0, 1] \), there holds
    \[
        B_{2} = \norm{ \VXX^{\frac{1 - \gamma}{2}} \, (\hWXX + \lambda \IX)^{-1 / 2} }_{\HX \to \HX} \le \sqrt{2} \lambda^{-\gamma / 2}.
    \]
\end{proposition}
\begin{proof}
    To bound \( B_{2} \), we decompose it as follows:
    \begin{align*}
        B_{2}
         & = \norm{ \VXX^{\frac{1 - \gamma}{2}} \, (\hWXX + \lambda \IX)^{-1 / 2} }_{\HX \to \HX}
        \overset{\text{(a)}}{\le} \norm{ \VXX^{1 - \gamma} \, (\hWXX + \lambda \IX)^{-1} }_{\HX \to \HX}^{1 / 2}                                                                     \\
         & \le \norm{ \VXX^{1 - \gamma} \, (\VXX + \lambda \IX)^{-1} }_{\HX \to \HX}^{1 / 2} \cdot \norm{ (\VXX + \lambda \IX) \, (\hWXX + \lambda \IX)^{-1} }_{\HX \to \HX}^{1 / 2} \\
         & \overset{\text{(b)}}{\le} \lambda^{-\gamma / 2} \cdot \norm{ (\VXX + \lambda \IX) \, (\hWXX + \lambda \IX)^{-1} }_{\HX \to \HX}^{1 / 2}.
    \end{align*}
    In step (a) we apply the Cordes inequality \citep[Lemma~5.1]{Cordes1987SpectralTL}; step (b) follows from the uniform bound
    \[
        \norm{ \VXX^{1 - \gamma} \, (\VXX + \lambda \IX)^{-1} }_{\HX \to \HX}^{1 / 2}
        \le \LE( \sup_{t \ge 0} \frac{t^{1 - \gamma}}{t + \lambda} \RI)^{1 / 2}
        \le \lambda^{-\gamma / 2}.
    \]

    It remains to bound the operator norm of \( (\VXX + \lambda \IX) \, (\hWXX + \lambda \IX)^{-1} \). Observe that
    \[
        \hWXX + \lambda \IX = \LE( \IX - (\VXX - \hWXX) \, (\VXX + \lambda \IX)^{-1} \RI)^{-1} (\VXX + \lambda \IX),
    \]
    which implies
    \[
        (\VXX + \lambda \IX) \, (\hWXX + \lambda \IX)^{-1}
        = \LE( \IX - (\VXX - \hWXX) \, (\VXX + \lambda \IX)^{-1} \RI)^{-1}.
    \]
    The bound for \( \norm{ \VXX - \hWXX }_{\HX \to \HX} \) is given in \proref{pro: B0}. Thus
    \begin{align*}
         & \eqspace \norm{ (\VXX - \hWXX) \, (\VXX + \lambda \IX)^{-1} }_{\HX \to \HX}                                                                                                                      \\
         & \le \norm{ \VXX - \hWXX }_{\HX \to \HX} \cdot \norm{ (\VXX + \lambda \IX)^{-1} }_{\HX \to \HX}                                                                                                   \\
         & \le \kapX^{2} \kapY \varDelta_{\eta} \cdot \lambda^{-1} \neta^{-\frac{\iota}{2 \iota + 2}} \log \frac{4}{\delta'} + 10 M_{\eta} \kapX^{2} \cdot \lambda^{-1} \nf^{-1 / 2} \log \frac{4}{\delta}.
    \end{align*}
    If the sample sizes \( \neta \), \( \nf \) and the regularization parameter \( \lambda \) satisfy
    \[
        \begin{cases}
            \displaystyle \lambda \neta^{\frac{\iota}{2 \iota + 2}} \ge 4 \kapX^{2} \kapY \varDelta_{\eta} \log \frac{4}{\delta'}, \\[1em]
            \displaystyle \lambda \nf^{1 / 2} \ge 40 M_{\eta} \kapX^{2} \log \frac{4}{\delta},
        \end{cases}
    \]
    then
    \[
        \norm{ (\VXX - \hWXX) \, (\VXX + \lambda \IX)^{-1} }_{\HX \to \HX} \le \frac{1}{4} + \frac{1}{4} = \frac{1}{2}.
    \]
    This ensures that the Neumann series converges:
    \begin{align*}
         & \eqspace \norm{ (\VXX + \lambda \IX) \, (\hWXX + \lambda \IX)^{-1} }_{\HX \to \HX}        \\
         & = \norm{ \LE( \IX - (\VXX - \hWXX) \, (\VXX + \lambda \IX)^{-1} \RI)^{-1} }_{\HX \to \HX}
        \le \frac{1}{1 - 1 / 2} = 2.
    \end{align*}

    Combining these estimates, we conclude
    \[
        B_{2} \le \lambda^{-\gamma / 2} \cdot \norm{ (\VXX + \lambda \IX) \, (\hWXX + \lambda \IX)^{-1} }_{\HX \to \HX}^{1 / 2}
        \le \sqrt{2} \lambda^{-\gamma / 2}. \qedhere
    \]
\end{proof}

The following proposition bounds \( B_{3} \) in \eqref{eq: B1 B2 B3 B4}.
\begin{proposition}
    \label{pro: B3}
    Assume that the bounds in \proref{pro: B0} hold. If the sample sizes \( \neta \), \( \nf \) and the regularization parameter \( \lambda \) satisfy
    \[
        \begin{cases}
            \displaystyle \lambda \neta^{\frac{\iota}{2 \iota + 2}} \ge 4 \kapX^{2} \kapY \varDelta_{\eta} \log \frac{4}{\delta'}, \\[1em]
            \displaystyle \lambda \nf^{1 / 2} \ge 40 M_{\eta} \kapX^{2} \log \frac{4}{\delta},
        \end{cases}
    \]
    then there holds
    \begin{align*}
        B_{3} & = \norm{ (\hWXX + \lambda \IX)^{1 / 2} \, \glam(\hWXX) \, (\hWXX \, \flam - \hwy) }_{\HX}                                                                                     \\
              & \le 20 E M_{\eta} (F \kapX \norm{ \psiA }_{\HX} + 2 M_{y}) \kapX \cdot \lambda^{-1 / 2} \nf^{-1 / 2} \log \frac{4}{\delta}                                                    \\
              & \eqspace + 2 E \kapX \kapY \LE( F \kapX \norm{ \psiA }_{\HX} + 2 M_{y} \RI) \varDelta_{\eta} \cdot \lambda^{-1 / 2} \neta^{-\frac{\iota}{2 \iota + 2}} \log \frac{4}{\delta'} \\
              & \eqspace + 2 \sqrt{2} E F \norm{ \psiA }_{\HX} \cdot \lambda^{r + \frac{1}{2}}.
    \end{align*}
\end{proposition}
\begin{proof}
    By adding and subtracting \( (\WXX \, \flam - \wy) \), we decompose \( B_{3} \) as
    \[
        B_{3} \le B_{3, 1} + B_{3, 2},
    \]
    where
    \begin{align*}
         & B_{3, 1} = \norm{ (\hWXX + \lambda \IX)^{1 / 2} \, \glam(\hWXX) \LE( (\hWXX \, \flam - \hwy) - (\WXX \, \flam - \wy) \RI) }_{\HX}, \\
         & B_{3, 2} = \norm{ (\hWXX + \lambda \IX)^{1 / 2} \, \glam(\hWXX) \, (\WXX \, \flam - \wy) }_{\HX}.
    \end{align*}
    These two terms are bounded separately as follows:
    \begin{enumerate}
        \item For \( B_{3, 1} \), by the filter function property \eqref{eq: filter E}, we have
              \[
                  \norm{ (\hWXX + \lambda \IX)^{1 / 2} \, \glam(\hWXX) }_{\HX \to \HX} \le 2 E \cdot \lambda^{-1 / 2}.
              \]
              This, together with the bounds from \proref{pro: B0}, yields
              \begin{equation}
                  \label{eq: B31}
                  \begin{aligned}
                      B_{3, 1}
                       & \le \norm{ (\hWXX + \lambda \IX)^{1 / 2} \, \glam(\hWXX) }_{\HX \to \HX} \cdot \norm{ (\hWXX \, \flam - \hwy) - (\WXX \, \flam - \wy) }_{\HX} \\
                       & \le 20 E M_{\eta} (F \kapX \norm{ \psiA }_{\HX} + 2 M_{y}) \kapX \cdot \lambda^{-1 / 2} \nf^{-1 / 2} \log \frac{4}{\delta}.
                  \end{aligned}
              \end{equation}

        \item Observe that
              \begin{align*}
                  \WXX \, \flam - \wy
                   & = \EE[(\bsx, y) \sim \pSXtY]{ \tetamu(y) (\flam(\bsx) - y) \, \KX(\cdot, \bsx) }              \\
                   & = \EE[(\bsx, y) \sim \pSXtY]{ (\tetamu(y) - \etaA(y)) (\flam(\bsx) - y) \, \KX(\cdot, \bsx) } \\
                   & \eqspace + \EE[(\bsx, y) \sim \pSXtY]{ \etaA(y) (\flam(\bsx) - y) \, \KX(\cdot, \bsx) }.
              \end{align*}
              Furthermore, since \( \fTA(\bsx) = \EE[y \sim \pTYmX]{ y \mid \bsx } \), we have
              \begin{align*}
                   & \eqspace \EE[(\bsx, y) \sim \pSXtY]{ \etaA(y) (\flam(\bsx) - y) \, \KX(\cdot, \bsx) } \\
                   & = \EE[(\bsx, y) \sim \pTXtY]{ (\flam(\bsx) - y) \, \KX(\cdot, \bsx) }
                  = \EE[\bsx \sim \pTX]{ (\flam(\bsx) - \fTA(\bsx)) \, \KX(\cdot, \bsx) }                  \\
                   & = \VXX \, (\flam - \fTA).
              \end{align*}
              Therefore, \( B_{3, 2} \) can be further decomposed as
              \begin{equation}
                  \label{eq: B32 step 1}
                  \begin{aligned}
                      B_{3, 2}
                       & \le 2 E \cdot \lambda^{-1 / 2} \cdot \norm{ \EE[(\bsx, y) \sim \pSXtY]{ (\tetamu(y) - \etaA(y)) (\flam(\bsx) - y) \, \KX(\cdot, \bsx) } }_{\HX} \\
                       & \eqspace + \norm{ (\hWXX + \lambda \IX)^{1 / 2} \, \glam(\hWXX) \, \VXX \, (\flam - \fTA) }_{\HX},
                  \end{aligned}
              \end{equation}
              where the factor \( 2 E \cdot \lambda^{-1 / 2} \) follows from bounding \( \norm{ (\hWXX + \lambda \IX)^{1 / 2} \, \glam(\hWXX) }_{\HX \to \HX} \).

              For the first term in \eqref{eq: B32 step 1}, recall that in the proof of \proref{pro: B0}, we derive
              \[
                  \sup_{(\bsx, y) \in \calX \times \calY} |\flam(\bsx) - y|
                  \le F \kapX \norm{ \psiA }_{\HX} + 2 M_{y}.
              \]
              This bound, together with \eqref{eq: eta bound}, give
              \begin{equation}
                  \label{eq: B32 step 2}
                  \begin{aligned}
                       & \eqspace \norm{ \EE[(\bsx, y) \sim \pSXtY]{ (\tetamu(y) - \etaA(y)) (\flam(\bsx) - y) \, \KX(\cdot, \bsx) } }_{\HX}                                \\
                       & \le \kapX \kapY \LE( F \kapX \norm{ \psiA }_{\HX} + 2 M_{y} \RI) \varDelta_{\eta} \cdot \neta^{-\frac{\iota}{2 \iota + 2}} \log \frac{4}{\delta'}.
                  \end{aligned}
              \end{equation}

              For the second term in \eqref{eq: B32 step 1}, we decompose it as follows:
              \begin{align*}
                  \cdots
                   & \le \norm{ (\hWXX + \lambda \IX)^{1 / 2} \, \glam(\hWXX) \, (\hWXX + \lambda \IX)^{1 / 2} }_{\HX \to \HX}                                         \\
                   & \eqspace \cdot \norm{ (\hWXX + \lambda \IX)^{-1 / 2} \, \VXX^{1 / 2} }_{\HX \to \HX} \cdot \norm{ \VXX^{1 / 2} \, (\flam - \fTA) }_{\HX}          \\
                   & \le 2 E \cdot \norm{ (\hWXX + \lambda \IX)^{-1 / 2} \, \VXX^{1 / 2} }_{\HX \to \HX} \cdot F \norm{ \psiA }_{\HX} \cdot \lambda^{r + \frac{1}{2}},
              \end{align*}
              where the last inequality follows from the filter function property \eqref{eq: filter E} and \proref{pro: B1} (with \( \gamma = 0 \)). It remains to bound \( \norm{ (\hWXX + \lambda \IX)^{-1 / 2} \, \VXX^{1 / 2} }_{\HX \to \HX} \). By the Cordes inequality \citep[Lemma~5.1]{Cordes1987SpectralTL}, we have
              \begin{align*}
                  \norm{ (\hWXX + \lambda \IX)^{-1 / 2} \, \VXX^{1 / 2} }_{\HX \to \HX}
                   & = \norm{ \VXX^{1 / 2} \, (\hWXX + \lambda \IX)^{-1 / 2} }_{\HX \to \HX} \\
                   & \le \norm{ \VXX \, (\hWXX + \lambda \IX)^{-1} }_{\HX \to \HX}^{1 / 2}.
              \end{align*}
              Moreover, since
              \[
                  (\VXX + \lambda \IX) \, (\hWXX + \lambda \IX)^{-1} - \VXX \, (\hWXX + \lambda \IX)^{-1} = \lambda (\hWXX + \lambda \IX)^{-1}
              \]
              is a positive operator, we obtain
              \[
                  \norm{ (\hWXX + \lambda \IX)^{-1} \, \VXX }_{\HX \to \HX}
                  \le \norm{ (\VXX + \lambda \IX) \, (\hWXX + \lambda \IX)^{-1} }_{\HX \to \HX}.
              \]
              As shown in \proref{pro: B2}, under the assumed conditions on \( \neta \), \( \nf \), and \( \lambda \), the right-hand side is bounded by \( 2 \). Combining these, we get
              \begin{equation}
                  \label{eq: B32 step 3}
                  \begin{aligned}
                       & \eqspace \norm{ (\hWXX + \lambda \IX)^{1 / 2} \, \glam(\hWXX) \, \VXX \, (\flam - \fTA) }_{\HX}                                                  \\
                       & \le 2 E \cdot \norm{ (\hWXX + \lambda \IX)^{-1 / 2} \, \VXX^{1 / 2} }_{\HX \to \HX} \cdot F \norm{ \psiA }_{\HX} \cdot \lambda^{r + \frac{1}{2}}
                      \le 2 \sqrt{2} E F \norm{ \psiA }_{\HX} \cdot \lambda^{r + \frac{1}{2}}.
                  \end{aligned}
              \end{equation}

              Substituting \eqref{eq: B32 step 2} and \eqref{eq: B32 step 3} into \eqref{eq: B32 step 1} yields a bound for \( B_{3, 2} \):
              \begin{equation}
                  \label{eq: B32}
                  \begin{aligned}
                      B_{3, 2}
                       & \le 2 E \kapX \kapY \LE( F \kapX \norm{ \psiA }_{\HX} + 2 M_{y} \RI) \varDelta_{\eta} \cdot \lambda^{-1 / 2} \neta^{-\frac{\iota}{2 \iota + 2}} \log \frac{4}{\delta'} \\
                       & \eqspace + 2 \sqrt{2} E F \norm{ \psiA }_{\HX} \cdot \lambda^{r + \frac{1}{2}}.
                  \end{aligned}
              \end{equation}
    \end{enumerate}

    Finally, combining \eqref{eq: B31} and \eqref{eq: B32}, we bound \( B_{3} \) as
    \begin{align*}
        B_{3}
         & \le B_{3, 1} + B_{3, 2}
        \le 20 E M_{\eta} (F \kapX \norm{ \psiA }_{\HX} + 2 M_{y}) \kapX \cdot \lambda^{-1 / 2} \nf^{-1 / 2} \log \frac{4}{\delta}                                                             \\
         & \hspace{7.5em} + 2 E \kapX \kapY \LE( F \kapX \norm{ \psiA }_{\HX} + 2 M_{y} \RI) \varDelta_{\eta} \cdot \lambda^{-1 / 2} \neta^{-\frac{\iota}{2 \iota + 2}} \log \frac{4}{\delta'} \\
         & \hspace{7.5em} + 2 \sqrt{2} E F \norm{ \psiA }_{\HX} \cdot \lambda^{r + \frac{1}{2}}.
    \end{align*}
    This completes the proof.
\end{proof}

The following proposition bounds the last term \( B_{4} \) in \eqref{eq: B1 B2 B3 B4}.
\begin{proposition}
    \label{pro: B4}
    Assume that the bounds in \proref{pro: B0} hold. Then we have
    \begin{align*}
        B_{4}
         & = \norm{ (\hWXX + \lambda \IX)^{1 / 2} \, (\IX - \glam(\hWXX) \, \hWXX) \, \flam }_{\HX}                                                                                                                                                                                                     \\
         & \le 2 E F \norm{ \psiA }_{\HX} \cdot \lambda^{r + \frac{1}{2}}                                                                                                                                                                                                                               \\
         & \eqspace + 20 (r + 1) E F M_{\eta}^{r} \kapX^{2 r} \kapY \norm{ \psiA }_{\HX} \varDelta_{\eta} \LE( \lambda^{1 / 2} \neta^{-\frac{\iota}{2 \iota + 2} \cdot \min \family{ r, 1 }} \log \frac{4}{\delta'} + \lambda^{1 / 2} \nf^{-\frac{\min \family{ r, 1 }}{2}} \log \frac{4}{\delta} \RI).
    \end{align*}
\end{proposition}
\begin{proof}
    We begin by recalling that \( \flam = \glam(\VXX) \, \VXX \, \fTA \) and \( \fTA = \VXX^{r} \, \psiA \):
    \begin{align*}
        B_{4}
         & = \norm{ (\hWXX + \lambda \IX)^{1 / 2} \, (\IX - \glam(\hWXX) \, \hWXX) \, \flam }_{\HX}                                            \\
         & = \norm{ (\hWXX + \lambda \IX)^{1 / 2} \, (\IX - \glam(\hWXX) \, \hWXX) \, \glam(\VXX) \, \VXX^{r + 1} \, \psiA }_{\HX}             \\
         & \le E \norm{ \psiA }_{\HX} \cdot \norm{ (\hWXX + \lambda \IX)^{1 / 2} \, (\IX - \glam(\hWXX) \, \hWXX) \, \VXX^{r} }_{\HX \to \HX}.
    \end{align*}
    In the last step, we apply the filter function property \eqref{eq: filter E} to obtain \( \norm{ \glam(\VXX) \, \VXX }_{\HX \to \HX} \le E \). To proceed, we add and subtract \( \hWXX^{r} \):
    \begin{align*}
         & \eqspace \norm{ (\hWXX + \lambda \IX)^{1 / 2} \, (\IX - \glam(\hWXX) \, \hWXX) \, \VXX^{r} }_{\HX \to \HX}                                         \\
         & \le \norm{ (\hWXX + \lambda \IX)^{1 / 2} \, (\IX - \glam(\hWXX) \, \hWXX) \, \hWXX^{r} }_{\HX \to \HX}                                             \\
         & \eqspace + \norm{ (\hWXX + \lambda \IX)^{1 / 2} \, (\IX - \glam(\hWXX) \, \hWXX) }_{\HX \to \HX} \cdot \norm{ \VXX^{r} - \hWXX^{r} }_{\HX \to \HX} \\
         & \le 2 F \cdot \lambda^{r + \frac{1}{2}} + 2 F \cdot \lambda^{1 / 2} \cdot \norm{ \VXX^{r} - \hWXX^{r} }_{\HX \to \HX},
    \end{align*}
    where the last inequality follows from the filter function property \eqref{eq: filter F}. Next, we bound \( \norm{ \VXX^{r} - \hWXX^{r} }_{\HX \to \HX} \) using \lemref{lem: A^c - B^c}. Under our assumptions, we have
    \[
        \max \family{ \norm{ \VXX }_{\HX \to \HX}, \norm{ \hWXX }_{\HX \to \HX} } \le M_{\eta} \kapX^{2},
    \]
    which implies
    \[
        \norm{ \VXX^{r} - \hWXX^{r} }_{\HX \to \HX} \le
        \begin{cases}
            \norm{ \VXX - \hWXX }_{\HX \to \HX}^{r},                            & r \in (0, 1]; \\
            r (M_{\eta} \kapX^{2})^{r - 1} \norm{ \VXX - \hWXX }_{\HX \to \HX}, & r > 1.
        \end{cases}
    \]
    The bound for \( \norm{ \VXX - \hWXX }_{\HX \to \HX} \) is given in \proref{pro: B0}:
    \begin{align*}
        \norm{ \VXX - \hWXX }_{\HX \to \HX}
         & \le \kapX^{2} \kapY \varDelta_{\eta} \cdot \neta^{-\frac{\iota}{2 \iota + 2}} \log \frac{4}{\delta'} + 10 M_{\eta} \kapX^{2} \cdot \nf^{-1 / 2} \log \frac{4}{\delta}.
    \end{align*}
    Consequently,
    \begin{align*}
         & \eqspace \norm{ \VXX^{r} - \hWXX^{r} }_{\HX \to \HX}                                                                                                                                                                                        \\
         & \le (r + 1) (M_{\eta} \kapX^{2})^{r - 1} \LE( \kapX^{2} \kapY \varDelta_{\eta} \cdot \neta^{-\frac{\iota}{2 \iota + 2}} \log \frac{4}{\delta'} + 10 M_{\eta} \kapX^{2} \cdot \nf^{-1 / 2} \log \frac{4}{\delta} \RI)^{\min \family{ r, 1 }} \\
         & \le 10 (r + 1) (M_{\eta} \kapX^{2})^{r} \kapY \varDelta_{\eta} \LE( \neta^{-\frac{\iota}{2 \iota + 2} \cdot \min \family{ r, 1 }} \log \frac{4}{\delta'} + \nf^{-\frac{\min \family{ r, 1 }}{2}} \log \frac{4}{\delta} \RI).
    \end{align*}
    Combining these estimates yields
    \begin{align*}
        B_{4}
         & \le E \norm{ \psiA }_{\HX} \cdot \norm{ (\hWXX + \lambda \IX)^{1 / 2} \, (\IX - \glam(\hWXX) \, \hWXX) \, \VXX^{r} }_{\HX \to \HX}                                                                                                                                                           \\
         & \le 2 E F \norm{ \psiA }_{\HX} \LE( \lambda^{r + \frac{1}{2}} + \lambda^{1 / 2} \cdot \norm{ \VXX^{r} - \hWXX^{r} }_{\HX \to \HX} \RI)                                                                                                                                                       \\
         & \le 2 E F \norm{ \psiA }_{\HX} \cdot \lambda^{r + \frac{1}{2}}                                                                                                                                                                                                                               \\
         & \eqspace + 20 (r + 1) E F M_{\eta}^{r} \kapX^{2 r} \kapY \norm{ \psiA }_{\HX} \varDelta_{\eta} \LE( \lambda^{1 / 2} \neta^{-\frac{\iota}{2 \iota + 2} \cdot \min \family{ r, 1 }} \log \frac{4}{\delta'} + \lambda^{1 / 2} \nf^{-\frac{\min \family{ r, 1 }}{2}} \log \frac{4}{\delta} \RI).
    \end{align*}
    This completes the proof of the theorem.
\end{proof}

We now proceed to prove \thmref{thm: regression function}.
\begin{proofof}{\thmref{thm: regression function}}
    By setting \( \delta' = \delta / 2 \), the density ratio error bound in \eqref{eq: eta bound} holds with probability at least \( 1 - \delta / 2 \):
    \[
        \sup_{y \in \calY} |\etaA(y) - \tetamu(y)| \le \kapY \varDelta_{\eta} \cdot \neta^{-\frac{\iota}{2 \iota + 2}} \log \frac{8}{\delta}.
    \]
    Similarly, applying \proref{pro: B0} with \( \delta / 2 \) in place of \( \delta \), the bounds in \proref{pro: B0} hold with probability at least \( 1 - \delta / 2 \). Consequently, if the sample sizes \( \neta \), \( \nf \) and the regularization parameter \( \lambda \) satisfy
    \[
        \begin{cases}
            \displaystyle \lambda \neta^{\frac{\iota}{2 \iota + 2}} \ge 4 \kapX^{2} \kapY \varDelta_{\eta} \log \frac{8}{\delta}, \\[1em]
            \displaystyle \lambda \nf^{1 / 2} \ge 40 M_{\eta} \kapX^{2} \log \frac{8}{\delta},
        \end{cases}
    \]
    then all bounds from Propositions~\ref{pro: B1}--\ref{pro: B4} hold simultaneously with probability at least \( 1 - \delta \):
    \begin{align*}
         & B_{1} \le F \norm{ \psiA }_{\HX} \cdot \lambda^{r + \frac{1 - \gamma}{2}},
        \quad B_{2} \le \sqrt{2} \lambda^{-\gamma / 2},                                                                                                                                                                                                                        \\
         & B_{3} \le 20 E M_{\eta} (F \kapX \norm{ \psiA }_{\HX} + 2 M_{y}) \kapX \cdot \lambda^{-1 / 2} \nf^{-1 / 2} \log \frac{8}{\delta}                                                                                                                                    \\
         & \eqspace + 2 E \kapX \kapY \LE( F \kapX \norm{ \psiA }_{\HX} + 2 M_{y} \RI) \varDelta_{\eta} \cdot \lambda^{-1 / 2} \neta^{-A} \log \frac{8}{\delta}                                                                                                                \\
         & \eqspace + 2 \sqrt{2} E F \norm{ \psiA }_{\HX} \cdot \lambda^{r + \frac{1}{2}},                                                                                                                                                                                     \\
         & B_{4} \le 2 E F \norm{ \psiA }_{\HX} \cdot \lambda^{r + \frac{1}{2}}                                                                                                                                                                                                \\
         & \eqspace + 20 (r + 1) E F M_{\eta}^{r} \kapX^{2 r} \kapY \norm{ \psiA }_{\HX} \varDelta_{\eta} \LE( \lambda^{1 / 2} \neta^{-A \cdot \min \family{ r, 1 }} \log \frac{8}{\delta} + \lambda^{1 / 2} \nf^{-\frac{\min \family{ r, 1 }}{2}} \log \frac{8}{\delta} \RI),
    \end{align*}
    where we write \( A = \iota / (2 \iota + 2) \) for brevity.

    We now choose the regularization parameter \( \lambda \) as
    \[
        \lambda =
        \begin{cases}
            \displaystyle \neta^{- \frac{A}{r + 1}}, & \text{if } \neta^{2 A} < \nf;   \\[1em]
            \displaystyle \nf^{-\frac{1}{2 r + 2}},  & \text{if } \neta^{2 A} \ge \nf.
        \end{cases}
    \]
    Under this choice, a direct calculation shows that
    \[
        \max \family{ \lambda^{-1 / 2} \nf^{-1 / 2}, \lambda^{-1 / 2} \neta^{-A}, \lambda^{1 / 2} \neta^{-A \cdot \min \family{ r, 1 }}, \lambda^{1 / 2} \nf^{-\frac{\min \family{ r, 1 }}{2}} } \le \lambda^{r + \frac{1}{2}}.
    \]
    Consequently, we obtain
    \[
        \norm{ \VXX^{\frac{1 - \gamma}{2}} \, (\fTA - \hflam) }_{\HX}
        \le B_{1} + B_{2} (B_{3} + B_{4})
        \le \varDelta_{f} \cdot \lambda^{r + \frac{1 - \gamma}{2}} \log \frac{8}{\delta},
    \]
    where
    \begin{equation}
        \label{eq: Delta_f}
        \begin{aligned}
            \varDelta_{f}
             & = ((4 + 2\sqrt{2}) E + 1) F \norm{ \psiA }_{\HX} + 20 \sqrt{2} E M_{\eta} (F \kapX \norm{ \psiA }_{\HX} + 2 M_{y}) \kapX \\
             & \eqspace + 2 \sqrt{2} E \kapX \kapY \LE( F \kapX \norm{ \psiA }_{\HX} + 2 M_{y} \RI) \varDelta_{\eta}                    \\
             & \eqspace + 40 \sqrt{2} (r + 1) E F M_{\eta}^{r} \kapX^{2 r} \kapY \norm{ \psiA }_{\HX} \varDelta_{\eta}
        \end{aligned}
    \end{equation}
    is a constant independent of \( \neta \), \( \nf \), or \( \delta \).

    Specializing to \( \gamma = 0 \) or \( 1 \) yields bounds for \( \norm{ \fTA - \hflam }_{\calL^{2}(\calX, \pTX)} \) and \( \norm{ \fTA - \hflam }_{\HX} \), respectively:
    \[
        \norm{ \fTA - \hflam }_{\calL^{2}(\calX, \pTX)}
        \le \varDelta_{f} \cdot \log \frac{8}{\delta}
        \cdot
        \begin{cases}
            \displaystyle \neta^{- A \frac{r + 1/2}{r + 1}}, & \text{if } \neta^{2 A} < \nf;   \\[1em]
            \displaystyle \nf^{-\frac{r + 1/2}{2 r + 2}},    & \text{if } \neta^{2 A} \ge \nf,
        \end{cases}
    \]
    and
    \[
        \norm{ \fTA - \hflam }_{\HX}
        \le \varDelta_{f} \cdot \log \frac{8}{\delta}
        \cdot
        \begin{cases}
            \displaystyle \neta^{- A \frac{r}{r + 1}}, & \text{if } \neta^{2 A} < \nf;   \\[1em]
            \displaystyle \nf^{-\frac{r}{2 r + 2}},    & \text{if } \neta^{2 A} \ge \nf.
        \end{cases}
    \]
    We have thus established the theorem.
\end{proofof}

\subsection{Auxiliary Lemmas}

The following lemmas are used throughout the proofs.
\begin{lemma}[\citealp{Caponnetto2007OptimalRR}, Proposition~2]
    \label{lem: concentration vector}
    Let \( \xi \) be a random variable taking values in a separable Hilbert space \( \mathscr{H} \). Assume that \( \norm{ \xi }_{\mathscr{H}} \le G \) holds almost surely for some constant \( G > 0 \). Then, for any i.i.d. sample \( \family{ \xi_{i} }_{i=1}^{n} \) and any \( \delta \in (0, 1) \),
    \[
        \norm{ \EE{ \xi } - \frac{1}{n} \sum_{i=1}^{n} \xi_{i} }_{\mathscr{H}}
        \le 10 U n^{-1/2} \log \frac{2}{\delta}
    \]
    holds with probability at least \( 1 - \delta \).
\end{lemma}

\begin{lemma}[\citealp{Blanchard2010OptimalLR}, Lemma~E.3]
    \label{lem: A^c - B^c}
    Let \( A \) and \( B \) be positive self-adjoint operators on a separable Hilbert space \( \mathscr{H} \) such that \( \max \family{ \norm{ A }_{\mathscr{H} \to \mathscr{H}}, \norm{ B }_{\mathscr{H} \to \mathscr{H}} } \le U \). Then, for any \( c \ge 0 \),
    \[
        \norm{ A^{c} - B^{c} }_{\mathscr{H} \to \mathscr{H}} \le
        \begin{cases}
            \norm{ A - B }_{\mathscr{H} \to \mathscr{H}}^{c},       & c \le 1;
            \smallskip                                                         \\
            c U^{c-1} \norm{ A - B }_{\mathscr{H} \to \mathscr{H}}, & c > 1.
        \end{cases}
    \]
\end{lemma}

\bibliography{reference}

\end{document}